\documentclass{article} %
\usepackage{iclr2027_conference,times}
\iclrfinalcopy

\usepackage{amsmath,amsfonts,bm}

\def\eqref#1{equation~\ref{#1}}

\def\1{\bm{1}}

\DeclareMathAlphabet{\mathsfit}{\encodingdefault}{\sfdefault}{m}{sl}
\SetMathAlphabet{\mathsfit}{bold}{\encodingdefault}{\sfdefault}{bx}{n}

\usepackage{hyperref}
\usepackage{url}
\usepackage{multirow}
\usepackage{booktabs}
\usepackage{xcolor}
\usepackage{placeins} %
\usepackage{xspace}

\usepackage{graphicx}
\usepackage{listings}
\usepackage{booktabs}
\usepackage{multirow}
\usepackage{enumitem}
\usepackage{subcaption}
\usepackage{caption}
\usepackage{framed}
\usepackage{fancybox}
\usepackage{tcolorbox}
\usepackage{wrapfig}

\usepackage{algorithm}
\usepackage{amssymb}
\usepackage{algpseudocode}

\usepackage{booktabs}
\usepackage{multirow} %
\usepackage{siunitx}
\colorlet{punct}{red!60!black}
\definecolor{background}{HTML}{EEEEEE}
\definecolor{delim}{RGB}{20,105,176}
\colorlet{numb}{magenta!60!black}

\lstdefinelanguage{json}{
    basicstyle=\scriptsize\ttfamily,
    showstringspaces=false,
    breaklines=true,
    frame=single,
    backgroundcolor=\color{white},
    literate=
     *{0}{{{\color{numb}0}}}{1}
      {1}{{{\color{numb}1}}}{1}
      {2}{{{\color{numb}2}}}{1}
      {3}{{{\color{numb}3}}}{1}
      {4}{{{\color{numb}4}}}{1}
      {5}{{{\color{numb}5}}}{1}
      {6}{{{\color{numb}6}}}{1}
      {7}{{{\color{numb}7}}}{1}
      {8}{{{\color{numb}8}}}{1}
      {9}{{{\color{numb}9}}}{1}
      {:}{{{\color{punct}{:}}}}{1}
      {,}{{{\color{punct}{,}}}}{1}
      {\{}{{{\color{delim}{\{}}}}{1}
      {\}}{{{\color{delim}{\}}}}}{1}
      {[}{{{\color{delim}{[}}}}{1}
      {]}{{{\color{delim}{]}}}}{1},
}

\usepackage{dramatist, etoolbox, enumitem}
\usepackage[framemethod=TikZ]{mdframed}
\makeatletter

\patchcmd{\speaker}{\item[#1\speaksdel]}{\item[\speaksfont#1]}{}{}
\patchcmd{\@character}{\item[#1\speaksdel]}{\item[\speaksfont#1]}{}{}
\makeatother

\title{JetSpec++: Foresight Parallel Drafting with Self Refinements}

\author{
    Zhaoxiang Feng\thanks{Equal contributions.}\textsuperscript{\hspace{1.5mm}}$^{1}$
    \hspace{1pt}
    Lanxiang Hu\footnotemark[1]\textsuperscript{\hspace{1.5mm}}$^{1}$
    \hspace{1pt}
    Haoran Yuan$^{2}$
    \hspace{1pt}
    Yu-Yang Qian$^{3}$
    \hspace{1pt}
    Jihui Yang$^{4}$
    \hspace{1pt}
    \And
    \hspace{1pt}
    Peng Zhao$^{3}$
    \hspace{1pt}
    Yibo Zhu$^{5}$
    \hspace{1pt}
    Tajana Rosing$^{1}$
    \hspace{1pt}
    Hao Zhang$^{1}$ \\
    $^1$UC San Diego
    \hspace{1pt}
    $^2$UIUC
    \hspace{1pt}
    $^3$Nanjing University
    \hspace{1pt}
    $^4$Meta
    \hspace{1pt}
    $^5$StepFun
}

\newcommand{\sysname}{JetSpec++\xspace}

\usepackage{amsthm}\newtheorem{proposition}{Proposition}\newtheorem{assumption}{Assumption}\theoremstyle{remark}\newtheorem{remark}{Remark}\theoremstyle{plain} %

\begin{document}

\maketitle

\begin{abstract}

Speculative decoding's effectiveness is governed by two quantities: the acceptance rate of drafted tokens and the cost of producing them. We show the two objectives can be co-optimized with algorithm-system co-design, and present \sysname, which retains a strong learned draft model while hiding its cost almost entirely. \sysname combines three components: intra-block self refinement, which trains the draft model to iteratively correct its own parallel predictions within a block of tokens; inter-block foresight drafting, which begins the next block before the current one is verified; and pipelined drafting and verification, which overlaps the two stages so drafting's marginal wall-clock cost approaches zero. %
In comparison with state-of-the-art speculative decoding methods, \sysname attains a wall-clock speedup of up to 7.45$\times$ over autoregressive (AR) decoding and 1.54$\times$ over DSpark, with up to 1,361 tokens per second for a single request on GB200, and an average acceptance length of 8.94 across seven benchmarks, while preserving the target model's output distribution. Our code and models are available at: \textcolor{blue}{https://github.com/hao-ai-lab/JetSpec}.

\end{abstract}

\section{Introduction}\label{sec:intro}

Speculative decoding (SD) is a widely adopted technique for parallelizing large language model (LLM) generation \citep{leviathan2023fast,chen2023accelerating}. Its effectiveness is determined by two factors, the acceptance rate and the drafting cost. Earlier work showed that the target model's hidden states of the verified context can align the draft model with the target distribution and thereby raise the acceptance rate \citep{li2024eagle,li2025eagle}. DFlash employs parallel drafting and has further reduced per-token drafting cost \citep{chen2026dflash}, providing bigger headroom for higher throughput \citep{cheng2026dspark} and lower latency \citep{hu2026jetspec}. A natural question to ask is whether we can further reduce drafting cost to zero on the critical path. 

One solution is to further push down the drafting cost. Cheap drafting usually relies on model-free techniques such as prompt lookup decoding \citep{saxena2023pld,somasundaram2024pld} or retrieval from the text generated so far \citep{oliaro2024suffixdecoding,hu2025sam}. However, model-free methods show high variance and low average acceptance rates, particularly when reusable patterns are sparse or fail to generalize. The question then becomes, can we retain the high acceptance rate of model-based drafting, while removing drafting cost from the draft-then-verify critical path?

The drafting cost can be removed from the critical path by overlapping drafting for the next round with verification in the current round (Figure~\ref{fig:1_intro-teaser}). However, 
for a learned draft model conditioned on the target model’s hidden states, the key challenge is that it must draft ahead before the next verified context becomes available~\citep{zhang2025swiftspec}. Accurately predicting this future context is therefore critical to maintaining draft quality.

\begin{figure}[t]
\centering
\includegraphics[width=\textwidth]{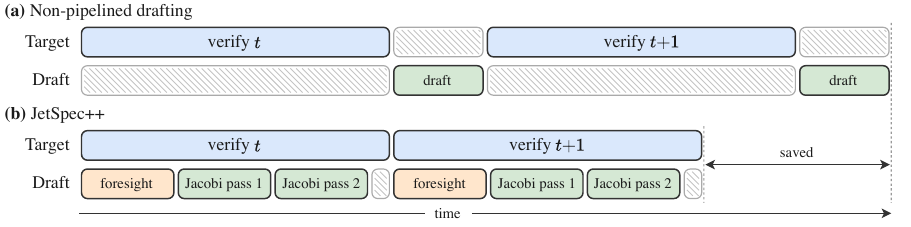}
\caption{\textbf{Non-pipelined drafting versus \sysname.} (a) A draft model that reads the target's hidden states drafts only after verification. (b) \sysname runs foresight and Jacobi passes while block $t$ is verified. Segment lengths are schematic.}
\label{fig:1_intro-teaser}
\end{figure}

We present \sysname, which maintains a target-conditioned parallel draft model and overlaps drafting for the next step with verification of the current step. To draft block $t{+}1$ while block $t$ is still being verified, however, the draft model must operate without two pieces of information that become available only after verification: (1) the number of accepted draft tokens in block $t$ and the corresponding bonus token, which becomes the \emph{anchor}, the first token of block $t{+}1$, and (2) the target hidden states for the newly verified context in block $t$.

\sysname addresses both missing components through \emph{Foresight Drafting}. Rather than waiting for the verification outcome, the draft model anticipates multiple possible ``world lines,'' each corresponding to a different acceptance-length outcome. For each request, a hidden-to-hidden adapter predicts the target hidden states of block $t$, which is still being verified. A lightweight decoder then maps these predicted states to possible anchor tokens. The predicted hidden states and anchor token are then used to draft block $t{+}1$.

Because drafting runs concurrently with target verification, additional draft-model computation incurs no extra latency as long as it remains within the verification window. \sysname exploits this otherwise idle compute by augmenting a causal parallel draft model built on the architecture of JetSpec~\citep{hu2026jetspec} with Jacobi-style self-refinement, allowing the draft model to iteratively revise its parallel predictions over multiple passes. Each additional refinement pass progressively moves the draft distribution toward a more autoregressive factorization with better target alignment. Between two consecutive passes of the draft model, \sysname additionally simulates verification with the adapter and the decoder as a proxy for the target. A \emph{context refresh} re-predicts the target states of the draft, an \emph{anchor replacement} replaces the first draft token that disagrees with the decoder, and the next pass refines the rest of the block (Section~\ref{sec:43}).
In summary:

\begin{itemize}
\item We propose \emph{Foresight Drafting}, which enables draft-verification pipelining for modern speculative decoding whose learned draft models depend on target-model hidden states. 
It anticipates possible verification outcomes and predicts the corresponding target hidden states and anchor tokens, thereby drafting the next block before the current block is verified.
\item We introduce Jacobi-style self-refinement with context refresh and anchor replacement for a causal parallel draft model\ that builds on JetSpec's architecture and is trained to refine its own draft (Section~\ref{sec:43}), exploiting otherwise hidden drafting compute within the verification window to progressively strengthen causal token dependence and improve target alignment over multiple passes.
\item We evaluate \sysname on seven benchmarks, where at verification budget 16 it attains the longest acceptance length of all methods on every benchmark at temperatures 0 and 1, with an average of 8.94 tokens per round against 6.89 for DSpark at temperature 0 (Table~\ref{tab:1}).
\item We integrate \sysname into a serving engine with CUDA graphs\ and the draft model on a second GPU, where its end-to-end speedup is up to 7.45$\times$ and 1.54$\times$ over AR and DSpark,
at up to 1,361 tokens per second for a single request, and \sysname cuts the per-round overhead beyond verification by up to 84\% relative to DSpark (Table~\ref{tab:2}).

\end{itemize}

\section{Related Work}\label{sec:related}

\textbf{Speculative Decoding.} Speculative decoding drafts several tokens and verifies them in one target forward pass, and its rejection-sampling rule leaves the output distribution unchanged~\citep{leviathan2023fast,chen2023accelerating}. Its speedup grows with the number of accepted tokens and shrinks with the cost of drafting them. To accept more tokens per pass, tree verification scores many candidate branches at once~\citep{
miao2024specinfer,chen2024sequoia,li2024eagle2}, 
and knowledge distillation aligns the draft model with the target~\citep{zhou2023distillspec,liu2024online_speculative_decoding,goel2024direct}. 
Medusa~\citep{cai2024medusa} predicts future positions from the last hidden state, 
and EAGLE~\citep{li2024eagle}, EAGLE-3~\citep{li2025eagle} and HASS~\citep{zhang2025hass} draft from the target's features.
To draft cheaply, parallel drafting reduces draft cost by predicting an entire block in one forward pass~\citep{an2026pard,xiao2024parallelspec}. Recent methods further improve parallel drafting alignment by conditioning on target hidden states~\citep{chen2026dflash}, but their parallel per-position predictions are not conditioned on the earlier realized draft tokens within the block~\citep{wang2026presto}. Sequential heads~\citep{cheng2026dspark,huang2026domino,zheng2026dels} restore this dependency at the cost of serial computation. DSpine~\citep{zhang2026dspine} instead writes each position's predicted feature into its successor after every layer, so the dependency forms through depth within one forward pass.
JetSpec~\citep{hu2026jetspec} introduces causal conditioning over hidden states\ within the block, while xPress~\citep{wang2026xpress} trains a separate refiner and refines the draft by Jacobi iteration. \sysname{} instead lets the draft model refine its own block. 

\textbf{Parallel Decoding.} Parallel decoding lets a model itself generate several of its own tokens per forward pass, without a separate draft model. Jacobi decoding updates every position of a block by fixed-point iteration and converges to the greedy output~\citep{song2020accelerating,santilli2023accelerating}. Lookahead decoding~\citep{fu2024break} collects $n$-grams from the Jacobi trajectory and verifies them, while consistency distillation~\citep{kou2024cllms} and Jacobi forcing~\citep{hu2025fast} train an AR model to reach the fixed point in fewer iterations. 
Discrete diffusion models generate text by denoising all positions in parallel over several steps~\citep{sahoo2024mdlm,nie2025llada,ye2025dream7b}. Block-wise discrete diffusion denoises one block at a time and conditions each block on the earlier ones~\citep{arriola2025block, wu2025fastdllmv2}, while caching~\citep{wu2025fast_dllm}, pseudo-trajectory distillation~\citep{qian2026d3llm} and multi-block decoding~\citep{wang2025diffusion_forcing} further accelerate diffusion inference. \sysname applies Jacobi iteration to the draft model rather than the target. 

\textbf{Draft-Verification Pipelining.}
Prior work hides drafting latency by running the draft model asynchronously with target verification~\citep{liu2024pearl,mcdanel2025amusd}.
Because the verification outcome is not yet known, recent methods hedge against multiple  outcomes: SpecBranch~\citep{shen2025specbranch} explores alternative branches, while speculative speculative decoding (SSD)~\citep{kumar2026speculative} prepares drafts for likely verification outcomes. SwiftSpec~\citep{zhang2025swiftspec}  uses asynchronous tree drafting to remove draft latency from the critical path. 
Target-conditioned draft models make pipelining harder because the representation of the context under verification is unavailable. Prior approaches obtain an early target-side signal from partial-depth computation~\citep{bhendawade2025mirror,venkatesha2025fast,yu2026speculative}, and SSD substitutes the draft model's own activations for it~\citep[Appendix~E]{kumar2026speculative}. D$^2$SD~\citep{zhang2026d} forms a distribution over the number of accepted tokens $k$ from the draft model's own confidences, and a second draft model redrafts the block from the most likely prefix lengths, without pipelining. DPara~\citep{liu2026dpara} pipelines a DSpark~\citep{cheng2026dspark} draft model without a fallback by precomputing its backbone for every possible $k$, so the backbone reads the tokens under verification without their target states. In contrast, \sysname predicts the missing target hidden states directly, enabling a target-conditioned parallel draft model to speculate ahead without 
executing the target model.

\section{Preliminaries}\label{sec:prelim}

An AR model generates one token per forward pass, and each pass is bound by memory bandwidth, %
so compute sits idle.  Speculative decoding (SD)~\citep{leviathan2023fast,chen2023accelerating} uses this idle compute to verify several drafted tokens in one forward pass. A cheap draft model $q$ proposes $\gamma$ tokens $d_1, \dots, d_\gamma$ after the verified context $x$, and the target model $p$ scores all of them in a single forward pass. From left to right, draft token $d_i$ is accepted with probability $\min\{1, p(d_i \mid x, d_{<i}) / q(d_i \mid x, d_{<i})\}$.  At the first rejection the rest of the draft is discarded and the target model samples a token of its own, the \textbf{bonus token}, from the residual distribution $\propto \max\{p - q, 0\}$. If every draft token is accepted, the bonus token is sampled from $p$ at position $\gamma + 1$. This procedure is lossless: it preserves the target model's generation distribution exactly.

Each round commits the accepted prefix plus one bonus token, and its expected number of committed tokens is the \textbf{acceptance length} $\tau$. A round costs one drafting step and one verification step, which take times $T_{\text{draft}}$ and $T_{\text{verify}}$.  The per-committed-token latency is therefore $L = (T_{\text{draft}} + T_{\text{verify}})/\tau$, so a more accurate draft model raises $\tau$ and a faster draft model lowers $T_{\text{draft}}$. 

\subsection{Parallel Drafting}\label{sec:31}

An AR draft model needs $\gamma$ sequential forward passes, so $T_{\text{draft}}$ grows with $\gamma$.  
Modern parallel draft models such as DFlash take the previous bonus token as the \textbf{anchor} $y_0$. From here on, $x$ denotes the verified context before $y_0$. Conditioned on the target hidden states $h_x$ of $x$, these models fill the other $\gamma$ positions with mask tokens, and predict all $\gamma$ positions  in one forward pass of the draft model $q_\theta$ with parameters $\theta$~\citep{chen2026dflash}: 

\begin{equation}
q_\theta(y_1, \dots, y_\gamma \mid x, h_x, y_0) \;=\; \prod_{i=1}^{\gamma} q_\theta(y_i \mid x, h_x, y_0)
\label{eq:1}
\end{equation}

$T_{\text{draft}}$ then barely depends on $\gamma$,  but each position is predicted without the tokens chosen at the other positions of the block. JetSpec~\citep{hu2026jetspec} makes the draft model causal, so in one forward pass every position attends to the hidden states at its own position and at the positions before it.
We build on this architecture to enable causal token factorization via multi-pass Jacobi iterations.

\subsection{Jacobi Decoding for Parallel Drafting}\label{sec:32}

Parallel decoding lets an AR LLM decode a $\gamma$-token block of its own greedy output, $y_i = \arg\max_v p(v \mid x, y_{<i})$ for $i = 1, \dots, \gamma$, by updating all $\gamma$ positions in  one forward pass. Jacobi decoding~\citep{song2020accelerating,santilli2023accelerating} formulates this process as solving these $\gamma$ equations by fixed-point iteration, updating every position from the previous iteration $r$:

\begin{equation}
y_i^{(r+1)} \;=\; \arg\max_{v}\; p\big(v \mid x,\, y^{(r)}_{<i}\big), \qquad i = 1, \dots, \gamma
\label{eq:2}
\end{equation}

Under a causal mask, each Jacobi iteration requires a single forward pass. The resulting iterates form a \textbf{Jacobi trajectory} that converges to the AR fixed point under greedy decoding~\citep{santilli2023accelerating}. 
\sysname trains the draft model on both fully noised inputs and partially converged blocks (Section~\ref{sec:43}).

\section{Methods}\label{sec:methods}

In this section, we describe how \sysname works. Draft-verification pipelining (Section~\ref{sec:41}) overlaps drafting with verification, reducing the critical-path latency for decoding and yielding the throughput model in Eq.~\ref{eq:4}. The challenge is that, in a pipelined round, the draft model must begin before the verification outcome is available. \sysname addresses this with \emph{Foresight Drafting} (Section~\ref{sec:42}), which learns predictive features that recover most of the acceptance-length loss (Appendix~\ref{app:adapter}). Finally, intra-block causal self-refinement (Section~\ref{sec:43}) makes use of the bubble exposed by pipelining to further increase the acceptance length. Between two passes of the draft model, a context refresh and an anchor replacement rerun the adapter and the decoder of Foresight Drafting to correct the draft.

\subsection{Draft-Verification Pipelining}\label{sec:41}

Draft-verification pipelining drafts block $t+1$ while block $t$ is being verified. 
In SSD's terms~\citep{kumar2026speculative}, the draft process anticipates the \textbf{verification outcome} $(k, a)$, the number $k$ of accepted draft tokens together with the bonus anchor token $a$, and fills a \textbf{speculation cache} with one pre-drafted block per anticipated outcome. When verification returns, the realized $(k, a)$ is matched. A \textbf{cache hit} makes the matching pre-drafted block the next block, and a \textbf{cache miss} falls back to an ordinary speculative step. Every committed token is verified by the rejection sampling acceptance rule~\citep{leviathan2023fast},
so the output distribution is unchanged.

\begin{figure}[t]
\centering
\includegraphics[width=\textwidth]{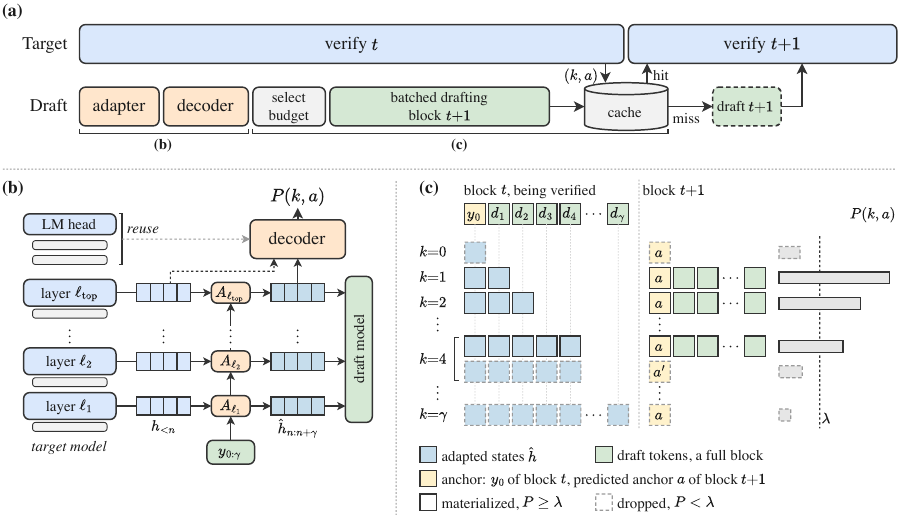}
\caption{
\textbf{Illustration of Foresight Drafting.} (a) 
\emph{Foresight Drafting} drafts the set $S$ of selected outcomes (Eq.~\ref{eq:8}) in one batch with $J{+}1$ passes.
(b) At each target layer $\ell$ that the draft model reads, the adapter attends to the verified prefix states $h^{\ell}_{<n}$ and predicts the current-block states $\hat h^{\ell}_{n:n+\gamma}$. A decoder converts the top-layer hidden state predictions into the outcome distribution $P(k,a)$ of Eq.~\ref{eq:6}. (c) The speculation cache for block $t{+}1$ contains outcome $k$, the draft tokens $d_1, \dots, d_k$ that outcome $k$ accepts, the corresponding adapted states $\hat h$ the draft model will read, and the decoded bonus token $a$ as the anchor of block $t{+}1$. An outcome whose calibrated score (Appendix~\ref{app:details}) reaches the threshold $\lambda$ will be pre-drafted as a $\gamma$-token block.
}
\label{fig:4_methods-1-overview}
\end{figure}

\emph{Foresight Drafting} (Section~\ref{sec:42}) predicts the  target hidden states that are missing until verification ends and estimates the distribution of the outcome from the prediction (Figure~\ref{fig:4_methods-1-overview}). The draft process of Figure~\ref{fig:4_methods-1-overview} (a) is hidden when

\begin{equation}
T_{\text{draft}} = T_{\text{pred}}(1) + R\, T_{\text{pred}}(|S|) + (J + 1)\, T_{\text{pass}}(|S|) 
+ T_{\text{host}} \;\le\; T_{\text{verify}}
\label{eq:3}
\end{equation}

where  $T_{\text{pred}}(\cdot)$ is the time of one  call of the adapter and the decoder  over the given number of blocks, one for the foresight call on the block under verification and $|S|$ for each context refresh on the drafts of the selected outcomes, $T_{\text{pass}}(|S|)$ is the time of one batched draft-model forward over the set $S$ of selected outcomes, $J$ is the number of Jacobi refinement passes, $R$ is the number of context refreshes (Section~\ref{sec:43}), 
and $T_{\text{host}}$ is the time of the CPU work between GPU phases. 
A choice of more selected outcomes raises $T_{\text{pass}}(|S|)$, so under Eq.~\ref{eq:3} it leaves fewer refinement passes $J$.

\textbf{Throughput Model.} 
Through draft-verification pipelining, on a hit the critical path overhead reduces to $\max(T_{\text{verify}}, T_{\text{draft}}) + T_{\text{comm}}$, where $T_{\text{comm}}$ is the communication time between the two processes.
On a miss, the draft model drafts a fallback block from the verified context in one pass (the initial draft $J_0$ of Section~\ref{sec:43}), which adds $T_{\text{fallback}}$ to the critical path, and the same round verifies it. %
With $p_{\text{hit}}$ being cache hit probability, and $\tau_{\text{hit}}$ and $\tau_{\text{miss}}$ being pre-drafted block and fallback block acceptance lengths, we adapt the speedup model of SSD~\citep{kumar2026speculative} and write

\begin{equation}
\text{throughput} \;=\; \frac{p_{\text{hit}}\, \tau_{\text{hit}} + (1 - p_{\text{hit}})\, \tau_{\text{miss}}}{\max\!\big(T_{\text{verify}},\, T_{\text{draft}}\big) + T_{\text{comm}} + (1 - p_{\text{hit}})\, T_{\text{fallback}}}
\label{eq:4}
\end{equation}

In the numerator, since $\tau_{\text{miss}}$ is smaller than $\tau_{\text{hit}}$, throughput grows with $p_{\text{hit}}$. Foresight drafting (Section~\ref{sec:42}) raises $p_{\text{hit}}$ while keeping $\tau_{\text{hit}}$ close to its value under true hidden states (Section~\ref{sec:exp}), and Jacobi iterations (Section~\ref{sec:43}) raise $\tau_{\text{hit}}$.

\textbf{Placement.} \sysname targets per-request latency. Because the two processes are separate, they can run on separate devices, as in SSD~\citep{kumar2026speculative}.

\subsection{Inter-Block Foresight Drafting}\label{sec:42}

Foresight drafting consists of an adapter, a decoder, and a draft-model backbone, where data flow is illustrated in Figure~\ref{fig:4_methods-1-overview}. %
Under each possible outcome $k$, positions $1, \dots, k$ of block $t$ hold the known draft tokens $d_1, \dots, d_k$, so the missing target-side hidden states belong to known tokens, and the state at each position $j \le k$ is predictable by a learned map from the tokens up to position $j$.
This allows training a specialized adapter to predict the hidden state at each position $j$ of block $t$. %

\textbf{Hidden-to-Hidden Adapter.} The adapter predicts, in one parallel pass, the target hidden states of all $\gamma + 1$ positions of the block under verification.
At each target layer $\ell$ that the draft model reads, the adapter attends to the cached prefix at that layer and predicts the hidden states of the new future positions:

\begin{equation}
\hat h^{\ell}_{n+j} \;=\; A_\ell\!\big(\, h^{\ell}_{< n},\; \hat h^{\ell^-}_{n:n+j} \,\big), \qquad j = 0, \dots, \gamma
\label{eq:5}
\end{equation}

Here $n$ is the position of the anchor $y_0$, the bonus token of the latest verification. $A_\ell$ is the adapter module for layer $\ell$, a short stack of transformer layers copied from the target (Appendix~\ref{app:adapter}).
$h^{\ell}_{< n}$ is the target's cached context states at layer $\ell$, the layer-$\ell$ slice of $h_x$ in Eq.~\ref{eq:1}, and  $\hat h^{\ell^-}$ is the adapter's prediction at the next lower target layer that the draft model reads, or the block's token embeddings for the lowest layer.
The adapter is trained separately from the draft model. During training, each module $A_\ell$ reads the prediction of the module below it $\hat h^{\ell^-}_{n:n+j}$, as at inference, and the adapter is trained using a Smooth L1 loss plus a cosine distance between its predictions and the target's true hidden states $h^{\ell}_{n:n+j}$, concatenated over the layers that the draft model reads
(Appendix~\ref{app:adapter}).

\textbf{Decoder.} The decoder turns the adapter's predictions into the outcome distributions. The decoder reuses frozen weights from the target's last few layers, above the deepest layer the adapter predicts, followed by the target's final norm and LM head. It attends to the verified context through keys and values computed from the target's true states (Appendix~\ref{app:details}).
The decoder then runs under a causal mask over the adapter's top-layer predictions $\hat h_n, \dots, \hat h_{n+\gamma}$. Position $j$ yields $\hat p(\cdot \mid j)$, an estimate of the target's next-token distribution after the prefix up to $j$. The estimate at position $j-1$ assigns draft token $d_j$ the probability $\hat p_j = \hat p(d_j \mid j-1)$. 
For each $k < \gamma$,

\begin{equation}
P(k) = \hat p_1 \cdots \hat p_k \,(1 - \hat p_{k+1}), \qquad P(k, a) \;=\; \hat p_1 \cdots \hat p_k \; \hat p(a \mid k), \quad a \neq d_{k+1}
\label{eq:6}
\end{equation}

If all $\gamma$ draft tokens are accepted, then $P(\gamma)=\prod_{j=1}^{\gamma}\hat p_j$ and $P(\gamma, a) = P(\gamma)\,\hat p(a \mid \gamma)$.
One decoder pass therefore scores every outcome at once, since row $j$ serves both the acceptance of $d_{j+1}$ and the bonus distribution at $k = j$ (Figure~\ref{fig:8_appendix-A2-outcome-matrix} in Appendix~\ref{app:decoder}). 
We call the decoder \emph{perfect} if its estimate $\hat p(\cdot \mid j)$ equals the target's next-token distribution at the sampling temperature after every prefix $j$ of the draft. With a perfect decoder, Eq.~\ref{eq:6} is the exact distribution of the outcome under speculative sampling at any temperature, for any draft that is fixed before verification, such as the Jacobi-refined argmax draft of Eq.~\ref{eq:9} (see Appendix~\ref{app:prop1}). In Appendix~\ref{sec:D1} (Table~\ref{tab:D1}), we show that the current design closely approximates the perfect decoder.

\textbf{Outcome Selection.} Given the outcome distributions, the draft process still needs to select the outcomes to pre-draft. 
Due to limited draft-side compute, the draft process selects, computes and caches a finite set of $B$ most promising outcomes $(k,a)$ of Eq.~\ref{eq:6} and $(\gamma,a)$. Outcomes are disjoint events, so the hit probability of a cache is the sum of its entries' probabilities:

\begin{equation}
S^\star \;=\; \arg\max_{|S| \le B} \; \sum_{(k,a) \in S} P(k, a)
\label{eq:8}
\end{equation}

The optimum is the $B$ most probable outcomes. In practice, decoder scores are imperfect, especially at longer lengths where errors compound through Eq.~\ref{eq:6}, so we calibrate the outcome scores on held-out prompts and retain candidates above a split-conformal threshold $\lambda$~\citep{angelopoulos2021a}, up to cache size $B$. The threshold is fit offline to include the realized outcome at a target rate on held-out rounds. At inference, this threshold can yield a variable number of cached outcomes across rounds. For simplicity, our experiments  instead  cache the $B$ highest-scoring outcomes in every round (Appendix~\ref{app:details}). Calibration affects only which outcomes are pre-drafted, never the verified output (Appendix~\ref{app:refresh}). Unlike SSD~\citep{kumar2026speculative}, which fixes how many anchors each value of $k$ receives, Eq.~\ref{eq:8} selects the cache dynamically from each round's outcome scores. With a perfect decoder, this selection is at least as good as any fixed allocation in every round (Proposition~\ref{prop:2}, Appendix~\ref{app:prop2}).

\textbf{Batched Drafting.} Having selected the outcomes to predict, each outcome $(k, a)$ is then drafted in the same way as an ordinary block. Its anchor is $a$, its context ends with $d_1, \dots, d_k$, and the draft model reads the adapter's predictions at the positions up to $k$ in place of the target's hidden states. The predicted states are the only approximation, and verification keeps the output lossless whatever they are. All selected outcomes are drafted as one batch, with the refinement passes of Section~\ref{sec:43}.

\subsection{Intra-Block Self Refinements}\label{sec:43}

A parallel draft model marginalizes each position over its possible predecessors, so its marginals are what PRESTO~\citep{wang2026presto} calls \emph{prefix-blind}. DSpark~\citep{cheng2026dspark} gives an example: when the context admits \emph{``of course''} and \emph{``no problem''}, the block may read \emph{``of problem''}. 
Sequential heads and separate refiners repair prefix-blindness at the cost of a sequential or recurrent module. We instead let the draft model refine its own draft in parallel, with no  module beyond those of Foresight Drafting, in the time that pipelining frees under verification. Self-refinement combines three operations, Jacobi iterations of the draft model, context refresh and anchor replacement.

\textbf{Jacobi Iterations.} Eq.~\ref{eq:1} factorizes the block into per-position marginals, whereas the target generates it token by token, $p(y_{1:\gamma} \mid x, y_0) = \prod_{i=1}^{\gamma} p(y_i \mid x, y_{<i})$. \sysname makes the draft model follow this causal, token-level factorization through Jacobi iterations. Our draft model takes \emph{natural-language tokens}, not mask embeddings, as the inputs, and attends causally within the block (Section~\ref{sec:31}), so each position can condition on the tokens drafted before it. Refinement is a Jacobi iteration (Section~\ref{sec:32}) of the draft model on its own draft:

\begin{equation}
y_i^{(r+1)} \;=\; \arg\max_{v}\; q_\theta\!\big(v \,\big|\, x,\; h_x,\; y^{(r)}_{\le i}\big), \qquad i = 1, \dots, \gamma
\label{eq:9}
\end{equation}

Here $y^{(r)}$ is the block after $r$ iteration passes, and $J_r$ names a round that stops at it, so $J$ in Eq.~\ref{eq:3} counts these passes. $J_0$ is the initial draft, whose inputs are a fully noised block of tokens drawn from the verified context, and $y_0$ is the anchor, the verified bonus token on a miss or the guessed $a$ of a pre-drafted block. Causal attention lets position $i$ read $y_{\leq i}$ for each iteration.
The block therefore moves from the factorization of Eq.~\ref{eq:1} toward an explicit causal token factorization, in which each position conditions on the drafted tokens $y_{<i}$ before it. %

\textbf{Context Refresh.} Jacobi iterations alone bring diminishing returns in acceptance length (Table~\ref{tab:6}, $R = 0$). Between two consecutive passes, \sysname therefore simulates verification on the current draft with the adapter and the decoder of Foresight Drafting. In a \emph{context refresh}, the adapter re-predicts the target hidden states of the current draft $y^{(r)}$. Block $t$ is still under verification, so both modules read the foresight predictions at its positions up to $k$ in place of true states. Like the states of Foresight Drafting, the refreshed states are predictions, because the target has not yet read the draft.

\textbf{Anchor Replacement.} The decoder then reads the refreshed states and finds the first position at which its most likely token differs from the draft. An \emph{anchor replacement} fixes the draft tokens before that position for the rest of the round, writes the decoder's token $\hat a$ at it, and fixes that position too. Like the anchor $a$ of a pre-drafted block, $\hat a$ is a guess of the token that the target would emit there. The next pass redrafts only the positions after the last fixed one (Appendix~\ref{app:refreshexp}). A round runs at most one context refresh, with its anchor replacement, between two consecutive passes, so it has $R \le J$ context refreshes. Both operations change only the proposal, so verification keeps the output lossless (Appendix~\ref{app:refresh}).

\textbf{Training.} A draft model trained only on masked blocks is never trained to condition on tokens at its inputs. Following Jacobi forcing~\citep{hu2025fast}, we train on blocks that pair a clean prefix with a noisy suffix, the shape of a partially converged Jacobi iterate. The first $m$ positions hold the true tokens and the rest hold noise tokens $\tilde y$, which we draw from the sequence's own context rather than from Jacobi trajectories. We train the draft model to match the target model at every position:

\begin{equation}
\mathcal L \;=\; \mathbb{E}_{m}\sum_{i=1}^{\gamma} \mathcal{D}\!\Big(p(\cdot \mid x, y_{<i}) \,\Big\|\, q_\theta\big(\cdot \mid x, h_x, [\,y_{0:m};\, \tilde y_{m+1:\gamma}\,]_{\le i}\big)\Big)
\label{eq:10}
\end{equation}

where $\mathcal{D}(\cdot\,\|\,\cdot)$ denotes a divergence. We use forward KL, which outperforms reverse KL and cross-entropy SFT in our ablations (Section~\ref{sec:54}).
The clean prefix length $m$ is sampled per block from a fixed schedule (Appendix~\ref{app:training}), so the model is trained to repair any clean-prefix block in Jacobi decoding. %
Appendix~\ref{app:training} provides our clean-prefix selection rule and details on our training setup.

\section{Experiments}\label{sec:exp}

\subsection{Setup}\label{sec:51}

\begin{table}[t]
\centering
\small
\setlength{\tabcolsep}{4pt}
\caption{Acceptance length $\tau$ at verification budgets 16 and 32 for Qwen3-8B, in the setup of Section~\ref{sec:51}. JetSpec++ runs $J_1$ with one context refresh and anchor replacement ($J_1$-$R_1$) and $J_3$ with three of each ($J_3$-$R_3$), and caches the $B = 16$ highest-scoring outcomes per round (Appendix~\ref{app:details}).}
\label{tab:1}
\resizebox{\ifdim\width>\linewidth\linewidth\else\width\fi}{!}{\begin{tabular}{lc*{7}{S[table-format=2.2]}}
\toprule
Method & Budget & {GSM8K} & {MATH-500} & {AIME25} & {HumanEval} & {MBPP} & {LCB} & {MT-Bench} \\
\midrule
\multicolumn{9}{c}{Temperature = 0} \\
\midrule
\multirow{2}{*}{DFlash} & 16 & 5.91 & 7.59 & 6.89 & 5.16 & 4.91 & 5.45 & 3.24 \\
 & 32 & 5.18 & 6.67 & 6.07 & 4.53 & 4.31 & 4.70 & 2.95 \\
\midrule
\multirow{2}{*}{DSpark} & 16 & 7.31 & 9.01 & 8.35 & 6.57 & 6.17 & 6.92 & 3.89 \\
 & 32 & 7.06 & 8.66 & 7.73 & 6.25 & 5.85 & 6.38 & 3.77 \\
\midrule
\multirow{2}{*}{Domino} & 16 & 6.91 & 8.74 & 8.09 & 6.06 & 5.90 & 6.27 & 3.74 \\
 & 32 & 6.25 & 7.96 & 7.37 & 5.61 & 5.29 & 5.92 & 3.49 \\
\midrule
\multirow{2}{*}{JetSpec} & 16 & 5.90 & 7.43 & 6.78 & 5.28 & 4.92 & 5.45 & 3.17 \\
 & 32 & 6.02 & 7.81 & 7.02 & 5.34 & 4.96 & 5.53 & 3.19 \\
\midrule
\multirow{2}{*}{JetSpec++, $J_1$-$R_1$} & 16 & 7.98 & 9.63 & 8.97 & 7.23 & 6.73 & 7.59 & 4.38 \\
 & 32 & 8.77 & 11.45 & 10.42 & 7.70 & 7.02 & 8.29 & 4.52 \\
\midrule
\multirow{2}{*}{JetSpec++, $J_3$-$R_3$} & 16 & 9.78 & 11.04 & 10.49 & 8.79 & 8.24 & 8.95 & 5.29 \\
 & 32 & 11.83 & 14.55 & 13.13 & 10.27 & 9.25 & 10.50 & 5.51 \\
\midrule
\multicolumn{9}{c}{Temperature = 1} \\
\midrule
\multirow{2}{*}{DFlash} & 16 & 5.19 & 5.37 & 4.25 & 4.26 & 4.24 & 4.78 & 2.81 \\
 & 32 & 4.70 & 4.85 & 4.16 & 3.97 & 3.81 & 4.20 & 2.64 \\
\midrule
\multirow{2}{*}{DSpark} & 16 & 6.10 & 6.13 & 4.69 & 5.32 & 5.18 & 6.05 & 3.26 \\
 & 32 & 5.99 & 6.03 & 4.67 & 5.08 & 4.91 & 5.54 & 3.17 \\
\midrule
\multirow{2}{*}{Domino} & 16 & 5.86 & 5.94 & 4.79 & 4.97 & 4.96 & 5.50 & 3.16 \\
 & 32 & 5.47 & 5.87 & 4.62 & 4.68 & 4.47 & 5.27 & 2.96 \\
\midrule
\multirow{2}{*}{JetSpec} & 16 & 5.24 & 5.45 & 4.30 & 4.39 & 4.41 & 4.95 & 2.81 \\
 & 32 & 5.34 & 5.62 & 4.30 & 4.47 & 4.29 & 4.66 & 2.76 \\
\midrule
\multirow{2}{*}{JetSpec++, $J_1$-$R_1$} & 16 & 6.88 & 6.50 & 5.25 & 5.84 & 5.53 & 6.55 & 3.51 \\
 & 32 & 7.33 & 7.24 & 5.48 & 6.07 & 5.80 & 7.01 & 3.59 \\
\midrule
\multirow{2}{*}{JetSpec++, $J_3$-$R_3$} & 16 & 8.10 & 7.18 & 5.59 & 6.50 & 6.31 & 7.21 & 4.07 \\
 & 32 & 9.41 & 8.53 & 5.77 & 7.33 & 6.95 & 7.81 & 4.09 \\
\bottomrule
\end{tabular}}
\end{table}

\textbf{Models and baselines.} The target models are Qwen3-8B~\citep{yang2025qwen3_technical_report} and Qwen3.8-27B~\citep{qwen2026qwen38}, both with thinking mode disabled. We compare \sysname with DFlash~\citep{chen2026dflash}, DSpark~\citep{cheng2026dspark}, Domino~\citep{huang2026domino}, JetSpec~\citep{hu2026jetspec} and SSD~\citep{kumar2026speculative}, the last with its Qwen3-0.6B draft model. For each target, all draft models except SSD's are trained on the same 800k sequences regenerated by that target, with the recipes of Appendix~\ref{app:training}. The draft models are trained with a block size of 16 (the anchor and $\gamma = 15$ draft positions), and during inference we test the checkpoints with both budget 16 and 32, the block size at verification. DSpark runs without its confidence-based scheduling, so every method verifies a block of fixed size.

\textbf{Benchmarks.} We evaluate on GSM8K~\citep{cobbe2021gsm8k}, MATH-500~\citep{hendrycks2021math}, AIME25~\citep{balunovic2025matharena}, HumanEval~\citep{chen2021evaluating}, MBPP~\citep{austin2021program}, LiveCodeBench (LCB)~\citep{jain2024livecodebench} and MT-Bench~\citep{zheng2023judging}, with the first 128 prompts of a fixed shuffle, all 30 AIME25 problems and the 80 two-turn MT-Bench questions (Appendix~\ref{app:training}).

\textbf{Evaluation protocol.} Every method generates at most 2,048 tokens, stopping at EOS or at the onset of a repetition loop (Appendix~\ref{app:training}). Tables~\ref{tab:1} and~\ref{tab:scale} measure the acceptance length offline, and Table~\ref{tab:2} measures serving at batch size 1 on NVIDIA GB200 GPUs with CUDA graphs. On a miss, \sysname drafts a fallback block from the verified context in one pass (Section~\ref{sec:41}). We report both greedy decoding and temperature 1 with rejection sampling (Appendix~\ref{app:refresh}). Speedup is measured in tokens per second (tok/s) relative to AR on one GPU of the same type. In Table~\ref{tab:2}, SSD and \sysname{} place their draft models on a second GPU (Section~\ref{sec:41}). The acceptance length $\tau$ is the mean number of tokens committed per verification round.

\subsection{Results}\label{sec:52}

Table~\ref{tab:1} reports the acceptance length. At both temperatures and both budgets, \sysname{} commits more tokens per round than every baseline on all seven benchmarks. At temperature 0 and budget 16, $J_1$-$R_1$ averages 7.50 against 6.89 for DSpark, the strongest baseline, and $J_3$-$R_3$ raises it to 8.94. Foresight drafting costs little acceptance length. At each of $J_0$ to $J_3$ without context refresh, the draft model keeps 95 to 98\% of the $\tau$ it reaches from the target's true states (Table~\ref{tab:F2}).

Although every draft model is trained on blocks of 16, \sysname{} extrapolates to blocks of 32. At budget 32, $J_1$-$R_1$ gains $\tau$ on every benchmark at both temperatures, whereas DFlash, DSpark and Domino lose it. On MATH-500 at temperature 0, for example, $\tau$ rises from 9.63 to 11.45 for $J_1$-$R_1$ and falls from 9.01 to 8.66 for DSpark. With causal attention inside the block, the 16 added positions leave the draft of the first 16 unchanged, and JetSpec, also causal, likewise gains $\tau$ at temperature 0. DFlash and Domino attend bidirectionally within the block, so the 16 added positions also change the draft of the first 16.

\subsection{Generalization to a Larger Target}\label{sec:scale}

\begin{table}[t]

\centering
\small
\setlength{\tabcolsep}{4pt}
\caption{Generalization to a larger target. Acceptance length $\tau$ for Qwen3.8-27B as in Table~\ref{tab:1}.}
\label{tab:scale}
\resizebox{\ifdim\width>\linewidth\linewidth\else\width\fi}{!}{\begin{tabular}{lc*{7}{S[table-format=2.2]}}
\toprule
Method & Budget & {GSM8K} & {MATH-500} & {AIME25} & {HumanEval} & {MBPP} & {LCB} & {MT-Bench} \\
\midrule
\multicolumn{9}{c}{Qwen3.8-27B, temperature = 0} \\
\midrule
\multirow{2}{*}{DSpark} & 16 & 8.02 & 8.94 & 8.30 & 8.19 & 7.62 & 5.19 & 4.00 \\
 & 32 & 7.57 & 8.49 & 7.64 & 7.13 & 6.96 & 4.84 & 3.83 \\
\midrule
\multirow{2}{*}{JetSpec++, $J_1$-$R_1$} & 16 & 8.71 & 9.53 & 8.58 & 9.18 & 8.41 & 5.80 & 4.30 \\
 & 32 & 9.68 & 11.10 & 9.82 & 10.72 & 9.42 & 5.94 & 4.40 \\
\midrule
\multirow{2}{*}{JetSpec++, $J_3$-$R_3$} & 16 & 10.05 & 10.83 & 10.15 & 10.72 & 9.80 & 6.93 & 4.91 \\
 & 32 & 12.35 & 14.05 & 12.49 & 13.47 & 11.76 & 7.43 & 5.15 \\
\midrule
\multicolumn{9}{c}{Qwen3.8-27B, temperature = 1} \\
\midrule
\multirow{2}{*}{DSpark} & 16 & 6.48 & 7.60 & 7.08 & 6.91 & 6.12 & 4.29 & 3.39 \\
 & 32 & 6.42 & 7.17 & 6.35 & 6.25 & 5.66 & 4.08 & 3.23 \\
\midrule
\multirow{2}{*}{JetSpec++, $J_1$-$R_1$} & 16 & 7.84 & 7.98 & 7.16 & 8.07 & 6.92 & 4.66 & 3.55 \\
 & 32 & 9.02 & 8.87 & 7.88 & 8.60 & 7.38 & 4.89 & 3.59 \\
\midrule
\multirow{2}{*}{JetSpec++, $J_3$-$R_3$} & 16 & 8.46 & 9.00 & 7.96 & 8.48 & 7.80 & 5.34 & 3.92 \\
 & 32 & 10.64 & 10.50 & 9.03 & 9.93 & 8.16 & 5.67 & 4.05 \\
\bottomrule
\end{tabular}}
\end{table}

Table~\ref{tab:scale} compares \sysname{} with DSpark on Qwen3.8-27B, a hybrid model in which 48 of the 64 layers are Gated DeltaNet~\citep{yang2025gated} layers and 16 are full-attention layers. The adapter and the decoder adapt to the hybrid layers (Appendices~\ref{app:adapter} and~\ref{app:details}). The adapter has 12 layers, as for Qwen3-8B (Table~\ref{tab:C3}), although the 27B target has 64 layers against 36. With this adapter, \sysname{} at $J_1$-$R_1$ leads DSpark on all seven benchmarks at temperature 0 and budget 16, with 8.71 and 9.53 on GSM8K and MATH-500 against 8.02 and 8.94. As on Qwen3-8B, the $\tau$ of \sysname{} rises at budget 32 on all seven benchmarks, whereas that of DSpark falls.

\subsection{System Performance}\label{sec:53}

Table~\ref{tab:2} measures how much drafting latency \sysname removes from each draft-verification round and how that translates into end-to-end speedup. We evaluate all methods under the same CUDA-graph serving setup (Appendix~\ref{app:serving}), closer to production inference. At $J_1$ with one context refresh and anchor replacement, \sysname is 1.54$\times$ faster than DSpark on GSM8K and 1.51$\times$ faster on MATH-500. It commits slightly more tokens per round (7.78 vs.\ 7.34 and 9.41 vs.\ 8.98), but most of the gain comes from hiding drafting latency. For an 8B target, graphed verification takes 6.1 to 6.3 ms, and the draft process runs one refinement pass and one context refresh within it. DSpark adds 3.76 and 3.80 ms of drafting per round, whereas \sysname exposes only 0.61 and 0.66 ms, with the realized outcome already pre-drafted in 92.9\% and 91.7\% of rounds. SSD, designed for larger targets, reaches only 1.52$\times$ and 1.55$\times$: its autoregressive draft model outlasts verification by 8.51 and 8.55 ms per round, and each round commits at most five tokens ($\tau$ of 4.06 and 4.20).
As target size grows, verification becomes more expensive relative to drafting, creating a larger latency window for additional refinement passes and context refreshes. On Qwen3.8-27B, graphed verification takes 27.4 to 27.8 ms, and the draft process runs three refinement passes and three context refreshes within it. \sysname{} then exposes 0.71 and 0.77 ms per round against 4.46 and 4.52 ms for DSpark, and runs 1.37$\times$ and 1.31$\times$ faster than DSpark on GSM8K and MATH-500. %

\begin{table}[h]

\centering
\small
\setlength{\tabcolsep}{3pt}
\caption{Throughput, acceptance length and round time at temperature 0 and budget 16 (5 for SSD, Appendix~\ref{app:training}), on NVIDIA GB200 GPUs with CUDA graphs (Section~\ref{sec:51}). JetSpec++ runs $J_1$ with one context refresh and anchor replacement for Qwen3-8B and $J_3$ with three of each for Qwen3.8-27B. Speedup is over AR. Hit is the share $p_{\text{hit}}$ of rounds whose outcome was pre-drafted, and $\tau_{\text{hit}}$ and $\tau_{\text{miss}}$ are the mean numbers of tokens committed on hit and miss rounds, all three over the rounds after the first of each request, which has no cache.  Round and verify are the mean times per round in ms, and exposed is their difference.}
\label{tab:2}
\resizebox{\ifdim\width>\linewidth\linewidth\else\width\fi}{!}{\begin{tabular}{lllccccccccc}
\toprule
Model & Benchmark & Method & GPUs & tok/s (speedup) & $\tau$ & hit & $\tau_{\text{hit}}$ & $\tau_{\text{miss}}$ & round & verify & exposed \\
\midrule
\multirow{10}{*}{Qwen3-8B} & \multirow{5}{*}{GSM8K} & AR & 1 & 182.3 (1.00$\times$) & 1.00 & N/A & N/A & N/A & 5.49 & 5.45 & 0.04 \\
 &  & DSpark & 1 & 744.3 (4.08$\times$) & 7.34 & N/A & N/A & N/A & 9.92 & 6.16 & 3.76 \\
 &  & DFlash & 1 & 620.1 (3.40$\times$) & 5.87 & N/A & N/A & N/A & 9.52 & 6.11 & 3.42 \\
 &  & SSD & 2 & 276.5 (1.52$\times$) & 4.06 & 95.7\% & 4.12 & 3.23 & 14.75 & 6.25 & 8.51 \\
 &  & JetSpec++ & 2 & 1148.8 (6.30$\times$) & 7.78 & 92.9\% & 8.20 & 3.45 & 6.84 & 6.23 & 0.61 \\
\cmidrule{2-12}
 & \multirow{5}{*}{MATH-500} & AR & 1 & 182.7 (1.00$\times$) & 1.00 & N/A & N/A & N/A & 5.47 & 5.44 & 0.04 \\
 &  & DSpark & 1 & 902.1 (4.94$\times$) & 8.98 & N/A & N/A & N/A & 9.95 & 6.16 & 3.80 \\
 &  & DFlash & 1 & 795.5 (4.35$\times$) & 7.65 & N/A & N/A & N/A & 9.63 & 6.16 & 3.47 \\
 &  & SSD & 2 & 282.5 (1.55$\times$) & 4.20 & 85.8\% & 4.32 & 3.47 & 14.88 & 6.33 & 8.55 \\
 &  & JetSpec++ & 2 & 1361.1 (7.45$\times$) & 9.41 & 91.7\% & 9.83 & 4.87 & 6.92 & 6.26 & 0.66 \\
\midrule
\multirow{6}{*}{Qwen3.8-27B} & \multirow{3}{*}{GSM8K} & AR & 1 & 60.3 (1.00$\times$) & 1.00 & N/A & N/A & N/A & 16.58 & 16.54 & 0.04 \\
 &  & DSpark & 1 & 249.6 (4.14$\times$) & 7.93 & N/A & N/A & N/A & 31.88 & 27.42 & 4.46 \\
 &  & JetSpec++ & 2 & 342.4 (5.68$\times$) & 9.67 & 89.0\% & 10.46 & 3.97 & 28.47 & 27.76 & 0.71 \\
\cmidrule{2-12}
 & \multirow{3}{*}{MATH-500} & AR & 1 & 59.3 (1.00$\times$) & 1.00 & N/A & N/A & N/A & 16.88 & 16.83 & 0.04 \\
 &  & DSpark & 1 & 279.8 (4.72$\times$) & 8.99 & N/A & N/A & N/A & 32.09 & 27.57 & 4.52 \\
 &  & JetSpec++ & 2 & 367.2 (6.20$\times$) & 10.45 & 90.0\% & 11.10 & 4.65 & 28.47 & 27.70 & 0.77 \\
\bottomrule
\end{tabular}}

\end{table}

\subsection{Ablation Studies}\label{sec:54}

\begin{table}[htb]

\centering
\small
\setlength{\tabcolsep}{3.5pt}
\caption{Self-refinement of the draft model. Acceptance length $\tau$ over refinement passes at temperature 0 for the two-epoch ablation draft model, trained with token or mask inputs, and on blocks with a clean prefix drawn by the prefix-length schedule of Appendix~\ref{app:training} or on fully noised blocks only. The target's true states and no context refresh, 164 prompts per benchmark, 256 tokens. Paired per prompt, every other row trails ours at $J_1$, $J_2$ and $J_3$ by at least 8 standard errors.}
\label{tab:3}
\resizebox{\ifdim\width>\linewidth\linewidth\else\width\fi}{!}{\begin{tabular}{lccccccccc}
\toprule
 &  & \multicolumn{4}{c}{GSM8K} & \multicolumn{4}{c}{MATH-500} \\
\cmidrule(lr){3-6}\cmidrule(lr){7-10}
Inputs & Training blocks & $J_0$ & $J_1$ & $J_2$ & $J_3$ & $J_0$ & $J_1$ & $J_2$ & $J_3$ \\
\midrule
\multirow{2}{*}{token (ours)} & clean prefix (ours) & 4.34 & 4.87 & 5.01 & 5.05 & 5.91 & 6.75 & 6.99 & 7.08 \\
 & fully noised & 4.36 & 4.37 & 4.37 & 4.37 & 6.02 & 6.01 & 6.02 & 6.01 \\
\midrule
\multirow{2}{*}{mask} & clean prefix & 4.37 & 4.57 & 4.63 & 4.67 & 5.93 & 6.29 & 6.47 & 6.56 \\
 & fully noised & 4.34 & 1.23 & 1.26 & 1.23 & 5.99 & 1.28 & 1.32 & 1.29 \\
\bottomrule
\end{tabular}}
\end{table}

Table~\ref{tab:3} isolates the draft model's side of \sysname against controls, and Table~\ref{tab:6} isolates context refresh and anchor replacement, each under the condition in its caption. Appendix~\ref{app:adapter} and Appendix~\ref{app:decoder} evaluate the adapter and the decoder.

\textbf{Refinement.} Table~\ref{tab:3} tests whether the draft model can refine its own draft. A draft model trained on token inputs with every block fully noised drafts as well at $J_0$ but does not refine, so refinement is learned from the blocks with a clean prefix, which the prefix-length schedule of Appendix~\ref{app:training} supplies. With this schedule but mask inputs, a draft model keeps only 42\% and 54\% of the gain from $J_0$ to $J_3$. At $J_0$ the other three draft models differ from ours by at most 0.11 tokens, so the inputs and the schedule decide how much a draft model refines. The final draft models show the same split at temperature 0 (Table~\ref{tab:E2}). The JetSpec draft model, which is trained on mask inputs and fully noised blocks, falls from 5.86 at $J_0$ to 2.39 at $J_1$ on GSM8K and from 7.51 to 2.57 on MATH-500. The JetSpec++ draft model instead rises from 5.70 at $J_0$ to 6.97 at $J_3$ on GSM8K and from 7.52 to 8.98 on MATH-500.

\textbf{Context Refresh and Anchor Replacement.} Table~\ref{tab:6} compares $J$ passes alone with the same passes and a context refresh before each, with or without anchor replacement. Refinement alone raises $\tau$ by 0.42 to 1.27 tokens from $J_0$ to $J_3$, and context refreshes with anchor replacement add 0.92 to 1.55 tokens at $J_1$-$R_1$ and 1.72 to 2.97 at $J_3$-$R_3$ over the same passes alone. Context refresh alone adds at most 0.84 tokens, at most 34\% of the gain with anchor replacement, so most of that gain comes from anchor replacement, not from the next pass reading the re-predicted states. Appendix~\ref{app:refreshexp} traces one context refresh and anchor replacement (Figure~\ref{fig:8_appendix-R1-refresh-correction}) and measures it without pipelining (Table~\ref{tab:E4}).

\begin{table}[t]

\centering
\small
\setlength{\tabcolsep}{3.5pt}
\caption{Gain of context refresh and anchor replacement. Acceptance length $\tau$ at temperature 0 and budget 16 for the draft model and in the setting of Table~\ref{tab:1}, with $J$ passes after $J_0$ and $R$ context refreshes per round. A context refresh alone re-predicts the target states of the draft tokens and keeps every draft token. With anchor replacement, it writes the decoder's token $\hat a$ of Eq.~\ref{eq:refresh} into the block.}
\label{tab:6}
\resizebox{\ifdim\width>\linewidth\linewidth\else\width\fi}{!}{\begin{tabular}{cclccccccc}
\toprule
$J$ & $R$ & Anchor repl. & GSM8K & MATH-500 & AIME25 & HumanEval & MBPP & LCB & MT-Bench \\
\midrule
0 & 0 & -- & 5.54 & 7.33 & 6.62 & 5.07 & 4.77 & 5.60 & 3.15 \\
1 & 0 & -- & 6.43 & 8.22 & 7.46 & 5.80 & 5.43 & 6.34 & 3.46 \\
2 & 0 & -- & 6.72 & 8.52 & 7.77 & 6.04 & 5.65 & 6.66 & 3.54 \\
3 & 0 & -- & 6.81 & 8.59 & 7.82 & 6.12 & 5.72 & 6.71 & 3.57 \\
\midrule
1 & 1 & no & 6.86 & 8.45 & 7.86 & 6.15 & 5.71 & 6.76 & 3.65 \\
1 & 1 & yes & 7.98 & 9.63 & 8.97 & 7.23 & 6.73 & 7.59 & 4.38 \\
\midrule
2 & 2 & no & 7.48 & 9.08 & 8.55 & 6.59 & 6.14 & 7.06 & 3.81 \\
2 & 2 & yes & 9.23 & 10.68 & 10.17 & 8.33 & 7.79 & 8.44 & 4.98 \\
\midrule
3 & 3 & no & 7.65 & 9.26 & 8.56 & 6.78 & 6.31 & 7.32 & 3.88 \\
3 & 3 & yes & 9.78 & 11.04 & 10.49 & 8.79 & 8.24 & 8.95 & 5.29 \\
\bottomrule
\end{tabular}}
\end{table}

\textbf{Training loss and attention.} Forward KL leads SFT and reverse KL in every cell, by up to 0.27 and 2.72 tokens (Table~\ref{tab:4}). Causal attention within the block matches bidirectional attention at $J_0$ and leads it at $J_3$, by 0.15 tokens on MATH-500 (Table~\ref{tab:5}). Both tables are in Appendix~\ref{app:drafttrain}.

\textbf{Adapter.} At $J_0$, the adapter's predicted states recover 95\% of the gap between copying the last verified state and the true states on both GSM8K and MATH-500 (Appendix~\ref{app:adapter}).

\subsection{Limitations}\label{sec:55}
Our serving integration can be further improved through systems optimization. In particular, fused kernels and optimized communication could reduce drafting overhead and enable additional refinement passes and context refreshes to further improve throughput.

\raggedbottom %
\section{Conclusion}

We presented \sysname, which drafts the next block while the current one is verified. Foresight Drafting predicts the target states and the verification outcome before verification returns. We train the causal parallel draft model on blocks with a clean prefix and a noisy suffix, so that self-refinement can spend the freed draft-side time on Jacobi passes and context refreshes. On GB200 GPUs, \sysname decodes up to 7.45$\times$ faster than AR decoding and 1.54$\times$ faster than DSpark.

\bibliography{jetspec++}
\bibliographystyle{iclr2027_conference}

\newpage
\appendix
\raggedbottom
\setlength{\lineskiplimit}{0pt}

\section{Adapter Design and Ablations}\label{app:adapter}

The predicted states of the adapter (Section~\ref{sec:42}) recover most of the acceptance length that the target's true states give (Sections~\ref{sec:52} and~\ref{sec:54}). In this appendix, we describe the adapter's architecture and training, measure this recovery, and ablate the number of adapter layers and the position weight in the loss.

\textbf{Architecture.} The adapter has one module $A_\ell$ of Eq.~\ref{eq:5} for each target layer $\ell$ that the draft model reads, and each module is a short stack of transformer layers. Module $A_\ell$ takes the block's states at the next lower layer $\ell^-$ that the draft model reads and predicts the block's states at $\ell$. The lowest module takes the block's token embeddings instead, so the modules form one chain from the tokens to the deepest layer that the draft model reads. In each module, the queries come from the block positions, and the keys and values come from the true states $h^{\ell}_{<n}$ of the verified context and the block's states $\hat h^{\ell^-}$ at earlier block positions. We initialize each module with copies of target layers between $\ell^-$ and $\ell$, ending with layer $\ell$. The adapter reuses the target's token embedding and keeps it frozen. We train all other weights of the adapter, and it adds no projection or head of its own. The adapter has 2.32B trainable parameters for Qwen3-8B and 4.47B for Qwen3.8-27B.

\textbf{Target layers.} The draft model reads five target layers. Layer indices start at 0. The modules for the three lowest layers each copy only their own layer, and the two deepest modules copy four and five layers, twelve layers in all.
\begin{itemize}[leftmargin=*,topsep=2pt,itemsep=2pt]
\item Qwen3-8B (36 layers): the draft model reads layers 1, 9, 17, 25 and 33, spaced uniformly from the second layer to the third-to-last as in DFlash~\citep{chen2026dflash}.
\item Qwen3.8-27B (64 layers): the draft model reads layers 5, 19, 33, 47 and 61. Each adapter layer for this target takes its MLP from the copied layer and its attention from the nearest full-attention layer at or below it.
\end{itemize}

\textbf{Training objective.} We train the adapter against the frozen target, which computes the true states online, on blocks of $\gamma+1 = 16$ positions. Let $\mathbf h_{n+j}$ be the concatenation of the target's true states $h^{\ell}_{n+j}$ at the layers that the draft model reads, and $\hat{\mathbf h}_{n+j}$ the adapter's prediction of it. For a block whose anchor is at position $n$, the loss is
\begin{equation}
\mathcal L_{\text{adapter}} \;=\; \sum_{j=0}^{\gamma} \frac{\omega_j^{\kappa}}{\sum_{j'=0}^{\gamma}\omega_{j'}^{\kappa}}\Big[\operatorname{SmoothL1}\big(\hat{\mathbf h}_{n+j},\mathbf h_{n+j}\big) + 1-\cos\big(\hat{\mathbf h}_{n+j},\mathbf h_{n+j}\big)\Big],\label{eq:B.2}
\end{equation}
where $\omega_j$ is the probability that verification reaches block position $j$, taken from served runs. The exponent $\kappa$ sets how the positions are weighted: $\kappa = 0$ weights every position equally, and $\kappa = 1$ weights each position by $\omega_j$. We use $\kappa = 1/2$.

\textbf{Training data and optimization.} We train on the 800k training sequences of the draft model (Appendix~\ref{app:training}), less 24k held out, with each sequence capped at 3,072 tokens. We use AdamW with a peak learning rate of $3 \times 10^{-4}$ (linear warmup over the first 4\% of steps, then cosine decay), gradient clipping at norm 1.0 and no weight decay. For Qwen3-8B, we train for six epochs at a global batch of 4 to 8 sequences, in 219 GPU-hours. For Qwen3.8-27B, we train for two epochs at a global batch of 8, with the AdamW $\beta_2$ set to 0.95, in 172 GPU-hours.

\textbf{Recovery of the true-state acceptance length.} Table~\ref{tab:F2} reports $\tau$ of the final draft model with stale, predicted and true states. At each of $J_0$ to $J_3$, $\tau$ from the predicted states is 95 to 98\% of $\tau$ from the true states. The predicted states close 95\% of the gap in $\tau$ between the stale and the true states at $J_0$ on GSM8K and MATH-500, and 89 to 95\% over all 28 cells.

\begin{table}[!tb]
\centering
\small
\setlength{\tabcolsep}{3.5pt}
\caption{Acceptance length $\tau$ of the final draft model for Qwen3-8B at budget 16, temperature 0 and no context refresh, with three sources of the target states at unverified positions. Stale states copy the last verified state into every new position, the final adapter predicts the states, and true states exist only after verification. Every round drafts only the block of the verified outcome, so no round misses. Prompts as in Table~\ref{tab:1} (first turn only for MT-Bench), with at most 2,048 generated tokens.}
\label{tab:F2}
\resizebox{\ifdim\width>\linewidth\linewidth\else\width\fi}{!}{\begin{tabular}{clccccccc}
\toprule
$J$ & States & GSM8K & MATH-500 & AIME25 & HumanEval & MBPP & LCB & MT-Bench \\
\midrule
\multirow{3}{*}{0} & stale & 2.89 & 3.62 & 3.47 & 2.70 & 2.58 & 3.17 & 1.93 \\
 & predicted & 5.56 & 7.31 & 6.89 & 4.98 & 4.68 & 5.93 & 2.98 \\
 & true & 5.70 & 7.52 & 7.14 & 5.09 & 4.81 & 6.11 & 3.09 \\
\midrule
\multirow{3}{*}{1} & stale & 3.36 & 4.17 & 4.00 & 3.13 & 2.97 & 3.64 & 2.11 \\
 & predicted & 6.41 & 8.26 & 7.79 & 5.67 & 5.33 & 6.65 & 3.28 \\
 & true & 6.59 & 8.55 & 8.16 & 5.85 & 5.49 & 6.89 & 3.40 \\
\midrule
\multirow{3}{*}{2} & stale & 3.55 & 4.35 & 4.17 & 3.27 & 3.10 & 3.79 & 2.17 \\
 & predicted & 6.67 & 8.52 & 8.07 & 5.90 & 5.59 & 6.84 & 3.34 \\
 & true & 6.88 & 8.86 & 8.47 & 6.12 & 5.71 & 7.12 & 3.48 \\
\midrule
\multirow{3}{*}{3} & stale & 3.63 & 4.41 & 4.25 & 3.33 & 3.14 & 3.85 & 2.19 \\
 & predicted & 6.76 & 8.63 & 8.19 & 5.99 & 5.66 & 6.90 & 3.37 \\
 & true & 6.97 & 8.98 & 8.58 & 6.22 & 5.79 & 7.17 & 3.51 \\
\bottomrule
\end{tabular}}
\end{table}

\begin{table}[!tb]
\centering
\small
\setlength{\tabcolsep}{3.5pt}
\caption{Effect of the number of adapter layers for Qwen3-8B. Parenthesized values indicate the layers per module for target layers 1, 9, 17, 25 and 33. We train each adapter for 0.46 epochs with $\kappa = 1$ in Eq.~\ref{eq:B.2}. (a) Error $\epsilon_\ell$ of Eq.~\ref{eq:C.1} (lower is better) and cosine similarity between the predicted and the true states at each layer $\ell$ that the draft model reads, on held-out mathematics. Top-1 is the fraction of block positions at which the decoder of Section~\ref{sec:42}, reading each adapter's predicted states, ranks the target's greedy token first. KL is $\mathrm{KL}(p \,\|\, \hat p)$ in nats, where $p$ is the target's next-token distribution at temperature 1 and $\hat p$ is the decoder's. Both use the 200 held-out mathematics sequences of Table~\ref{tab:D1}. (b) Acceptance length $\tau$ of the ablation draft model at temperature 0.}
\label{tab:C3}
\resizebox{\ifdim\width>\linewidth\linewidth\else\width\fi}{!}{\begin{tabular}{lcccccccccccc}
\multicolumn{13}{l}{(a) Predicted states} \\
\toprule
 & \multicolumn{5}{c}{Error $\epsilon_\ell$} & \multicolumn{5}{c}{Cosine} & \multicolumn{2}{c}{Decoder} \\
\cmidrule(lr){2-6}\cmidrule(lr){7-11}\cmidrule(lr){12-13}
Layers & 1 & 9 & 17 & 25 & 33 & 1 & 9 & 17 & 25 & 33 & Top-1 & KL \\
\midrule
15 (2/2/2/4/5) & 0.260 & 0.255 & 0.274 & 0.219 & 0.160 & 0.962 & 0.963 & 0.958 & 0.970 & 0.981 & .958 & .068 \\
12 (1/1/1/4/5) (ours) & 0.317 & 0.275 & 0.291 & 0.223 & 0.161 & 0.943 & 0.957 & 0.953 & 0.969 & 0.981 & .956 & .072 \\
10 (1/1/1/3/4) & 0.317 & 0.278 & 0.295 & 0.231 & 0.172 & 0.943 & 0.956 & 0.951 & 0.967 & 0.979 & .952 & .087 \\
8 (1/1/1/2/3) & 0.321 & 0.283 & 0.302 & 0.244 & 0.182 & 0.941 & 0.955 & 0.949 & 0.964 & 0.977 & .945 & .111 \\
\bottomrule
\end{tabular}}

\medskip
\resizebox{\ifdim\width>\linewidth\linewidth\else\width\fi}{!}{\begin{tabular}{lcccccccc}
\multicolumn{9}{l}{(b) Acceptance length $\tau$} \\
\toprule
 & \multicolumn{4}{c}{GSM8K} & \multicolumn{4}{c}{MATH-500} \\
\cmidrule(lr){2-5}\cmidrule(lr){6-9}
Layers & $J_0$ & $J_1$ & $J_2$ & $J_3$ & $J_0$ & $J_1$ & $J_2$ & $J_3$ \\
\midrule
15 (2/2/2/4/5) & 4.23 & 4.70 & 4.82 & 4.86 & 5.77 & 6.52 & 6.72 & 6.79 \\
12 (1/1/1/4/5) (ours) & 4.24 & 4.71 & 4.84 & 4.88 & 5.74 & 6.45 & 6.67 & 6.75 \\
10 (1/1/1/3/4) & 4.22 & 4.68 & 4.80 & 4.84 & 5.72 & 6.44 & 6.68 & 6.74 \\
8 (1/1/1/2/3) & 4.19 & 4.66 & 4.78 & 4.81 & 5.66 & 6.38 & 6.59 & 6.67 \\
\bottomrule
\end{tabular}}
\end{table}

\textbf{Ablation setup.} We ablate the number of adapter layers and the exponent $\kappa$ of Eq.~\ref{eq:B.2} for Qwen3-8B. To measure $\tau$, every adapter feeds the same frozen draft model, the ablation draft model of Appendix~\ref{app:training}. We train the adapters on the mathematics slice of the 800k set, apart from a held-out part. At each layer $\ell$ that the draft model reads, the error $\epsilon_\ell$ of an adapter is the relative error of its predicted states on 2,048 blocks drawn from the held-out sequences,
\begin{equation}
\epsilon_\ell \;=\; \frac{\sum_{n}\sum_{j} \omega_j^2 \,\lVert \hat h^{\ell}_{n+j} - h^{\ell}_{n+j} \rVert_2 \,/\, \lVert h^{\ell}_{n+j} \rVert_2}{\sum_{n}\sum_{j} \omega_j^2},
\label{eq:C.1}
\end{equation}
where $n$ runs over the anchors of these blocks and $j$ is the position within the block. We measure $\tau$ with the protocol of the ablation draft models (Appendix~\ref{app:training}) and no context refresh. The draft model reads each adapter's predicted states at the unverified positions.

\textbf{Number of adapter layers.} We vary the number of adapter layers in Table~\ref{tab:C3}. From fifteen to twelve layers, the modules for layers 1, 9 and 17 drop from two layers to one. This change raises their error (from 0.260 to 0.317 at layer 1) but changes the errors at layers 25 and 33 by at most 0.004 and $\tau$ at $J_3$ by at most 0.04 tokens. Removing layers from the modules for layers 25 and 33 then raises the error there step by step and lowers the decoder's agreement with the target. With eight layers instead of fifteen, $\tau$ at $J_3$ drops by 0.05 tokens on GSM8K and 0.12 on MATH-500. We therefore give the modules for layers 1, 9 and 17 one layer each and put the extra layers in the modules for layers 25 and 33. Each layer adds to the time of every adapter call in Eq.~\ref{eq:3}. We use twelve layers.

\textbf{Position weighting in the loss.} We train twelve-layer adapters on the split above and vary only the exponent $\kappa$ of Eq.~\ref{eq:B.2} (Table~\ref{tab:C2}). The $\kappa = 1/2$ adapter is the ablation adapter of Appendix~\ref{app:decoder}. The exponent has little effect on $\tau$, and at $J_3$ the four adapters lie within 0.03 tokens on both benchmarks. The value $\kappa = 1/2$ gives the lowest error at layers 25 and 33, and $\kappa = 1$ gives the lowest decoder agreement and the highest KL.

\begin{table}[!tb]
\centering
\small
\setlength{\tabcolsep}{3.5pt}
\caption{Effect of the exponent $\kappa$ of the position weight in the adapter loss (Eq.~\ref{eq:B.2}) for Qwen3-8B, with twelve-layer adapters trained for 1.83 epochs. $\epsilon_{25}$ and $\epsilon_{33}$ are the errors of Eq.~\ref{eq:C.1} at layers 25 and 33 on held-out mathematics (lower is better), and the $J_0$ to $J_3$ columns give the acceptance length $\tau$ of the ablation draft model at temperature 0. Top-1 and KL are as in Table~\ref{tab:C3}.}
\label{tab:C2}
\resizebox{\ifdim\width>\linewidth\linewidth\else\width\fi}{!}{\begin{tabular}{lcccccccccccc}
\toprule
 & & & \multicolumn{4}{c}{GSM8K} & \multicolumn{4}{c}{MATH-500} & \multicolumn{2}{c}{Decoder} \\
\cmidrule(lr){4-7}\cmidrule(lr){8-11}\cmidrule(lr){12-13}
$\kappa$ & $\epsilon_{25}$ & $\epsilon_{33}$ & $J_0$ & $J_1$ & $J_2$ & $J_3$ & $J_0$ & $J_1$ & $J_2$ & $J_3$ & Top-1 & KL \\
\midrule
$0$ & 0.212 & 0.149 & 4.28 & 4.75 & 4.88 & 4.93 & 5.84 & 6.59 & 6.82 & 6.91 & .967 & .043 \\
$1/4$ & 0.210 & 0.148 & 4.27 & 4.75 & 4.88 & 4.93 & 5.81 & 6.57 & 6.81 & 6.88 & .967 & .045 \\
$1/2$ (ours) & 0.208 & 0.147 & 4.28 & 4.76 & 4.89 & 4.95 & 5.86 & 6.59 & 6.82 & 6.89 & .967 & .043 \\
$1$ & 0.209 & 0.148 & 4.27 & 4.75 & 4.87 & 4.92 & 5.82 & 6.55 & 6.80 & 6.88 & .964 & .050 \\
\bottomrule
\end{tabular}}
\end{table}

\FloatBarrier

\section{Decoder and Outcome Scores}\label{app:decoder}

We describe the served configuration of the decoder of Section~\ref{sec:42}, which reuses the target's upper layers without training. Figure~\ref{fig:8_appendix-A2-outcome-matrix} shows how one decoder pass supplies every factor of the outcome score in Eq.~\ref{eq:6}. We then calibrate this score, compare the decoder with a fine-tuned copy, and compare the calibrated score with other selection rules.

\begin{figure}[!tb]
\centering
\includegraphics[width=\textwidth]{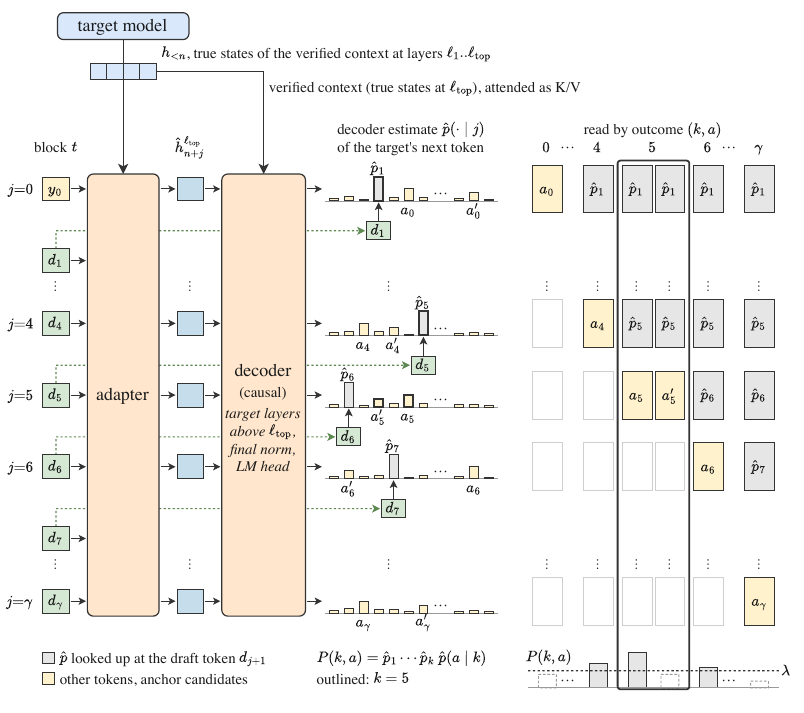}
\caption{\textbf{Scoring every outcome with one decoder pass.} The adapter predicts the block's states from the verified context, and the decoder reads them with the verified context as keys and values. Layers $\ell_1, \dots, \ell_{\text{top}}$ are the target layers that the draft model reads, and $\ell_{\text{top}}$ is the deepest of them. Row $j$ is the decoder's $\hat p(\cdot \mid j)$. For the outcomes with $k = 5$ (outlined), rows 0 to 4 supply the acceptance probabilities $\hat p_1, \dots, \hat p_5$ of the draft tokens, row 5 supplies the bonus distribution with $d_6$ excluded, and later rows are not read. The bars give each outcome's score against the threshold $\lambda$\ of Section~\ref{sec:42}: solid bars reach $\lambda$ and are cached, and dashed bars are dropped. Heights are schematic. Our experiments instead cache the $B$ outcomes with the highest calibrated score\ in every round (Appendix~\ref{app:details}).}
\label{fig:8_appendix-A2-outcome-matrix}
\end{figure}

\subsection{Scoring Rule and Cache Size}\label{app:details}

\textbf{Reused target layers.} The decoder reuses the target layers above the deepest layer that the adapter predicts, followed by the target's final norm and LM head, all frozen.\ Its estimate $\hat p(\cdot \mid j)$ is the softmax of the LM head's logits at temperature 1, at both sampling temperatures. For Qwen3-8B, the decoder reuses the two target layers above layer 33 (layers 34 and 35). The target sends its layer-33 states to the draft process after each verification. For Qwen3.8-27B, the decoder reuses the two target layers above layer 61 (layers 62 and 63). After each verification, the target sends the committed recurrent and convolution states of layer 62 (a Gated DeltaNet layer) and the keys and values of layer 63.

\textbf{Calibrated score.} \sysname{} scores each outcome $(k,a)$ by the \emph{calibrated score} $\tilde P(k)\,P_\beta(a \mid k)$. To compute $\tilde P(k)$, \sysname{} replaces each factor $\hat p_i$ of Eq.~\ref{eq:6} with a logistic function $\tilde p_i$ of a feature vector $\boldsymbol\phi_i$,
\begin{equation}
\tilde p_i \;=\; \sigma\big(\mathbf w^{\top} \boldsymbol\phi_i\big), \qquad G(k) = \prod_{i \le k} \tilde p_i, \qquad \tilde P(k) = G(k) - G(k+1), \quad G(\gamma+1) = 0.\label{eq:B.1}
\end{equation}
Here $\mathbf w$ holds one coefficient for each feature. The product $G(k)$ estimates the probability that verification accepts the first $k$ draft tokens, so $\tilde P(k)$ estimates the probability that verification accepts exactly $k$. The vector $\boldsymbol\phi_i$ holds six features of position $i$:
\begin{itemize}[leftmargin=*,topsep=2pt,itemsep=2pt]
\item the logit of the draft model's top-1 probability at position $i$ in its final pass,
\item the logit of the decoder's probability $\hat p_i$ of $d_i$,
\item an indicator of whether $d_i$ is the decoder's top token $\arg\max_v \hat p(v \mid i-1)$,
\item the logit of the decoder's largest renormalized probability $\max_{a \ne d_i} \hat p(a \mid i-1)/(1-\hat p_i)$ of a token other than $d_i$,
\item the scaled position $i/14$,
\item a constant bias feature.
\end{itemize}
The factor $P_\beta(a \mid k) \propto \hat p(a \mid k)^{1/\beta}$ is the decoder's bonus distribution at $k$, tempered by the anchor temperature $\beta$. \sysname{} renormalizes this distribution over the decoder's top 32 candidates, with $d_{k+1}$ removed when $k<\gamma$. We fit $\mathbf w$ and $\beta$ on the calibration rounds of Appendix~\ref{sec:D4}.

\textbf{Cache size.} For both targets, the draft process caches the $B$ highest-scoring outcomes per round, or all distinct outcomes when fewer exist. The cache size $B$ is 16 in Tables~\ref{tab:1} and~\ref{tab:scale}, and 4 in Table~\ref{tab:2}.  In the offline replays of Tables~\ref{tab:D2} and~\ref{tab:D3}, the served rule caches outcomes in the same way at each $B$.

\subsection{Fitting the Calibrated Score}\label{sec:D4}

We collect the calibration rounds by serving a pooled set of 56 held-out prompts for each target. For Qwen3-8B, each of GSM8K, MATH-500, HumanEval, MBPP and LCB contributes 8 prompts, drawn after its evaluation prompts. UltraChat~\citep{ding2023enhancing} and the GSM8K training split contribute 8 prompts each. We fit once per target and verification budget, and the same fit serves every benchmark and both temperatures.

The served rule has three fitted parts: the six coefficients $\mathbf w$ of Eq.~\ref{eq:B.1}, the anchor temperature $\beta$, and the split-conformal thresholds on the calibrated score. We fit $\mathbf w$ by maximum likelihood of the realized $k$ on the calibration rounds, and $\beta$ by maximum likelihood of the realized anchor. We set the thresholds on the same rounds by split-conformal prediction.

\subsection{Training-Free Versus Fine-Tuned Decoders}\label{sec:D1}

To test whether fine-tuning improves the reused layers, we compare the decoder of Appendix~\ref{app:details}, which we call the \emph{training-free decoder}, with a fine-tuned copy. We fine-tune the \emph{CE-trained decoder} for 6,000 updates on the adapter's predictions, with cross-entropy to the target's greedy token. Table~\ref{tab:D1} reports top-1 agreement, the fraction of block positions where the decoder's most probable token equals the target's greedy token.

\begin{table}[!tb]
\centering
\small
\caption{Offline agreement of each decoder with the greedy token of the Qwen3-8B target on 200 held-out mathematics sequences (379,305 positions), with the ablation adapter. KL is $\mathrm{KL}(p \,\|\, \hat p)$ in nats, where $p$ is the target's next-token distribution\ at temperature 1 and $\hat p$ is the decoder's.}
\label{tab:D1}
\resizebox{\ifdim\width>\linewidth\linewidth\else\width\fi}{!}{\begin{tabular}{lcc}
\toprule
Decoder & Top-1 agreement & KL to target \\
\midrule
training-free & .967 & .043 \\
CE-trained & .968 & .247 \\
\bottomrule
\end{tabular}}
\end{table}

Fine-tuning adds little to the reused layers (Table~\ref{tab:D1}). Cross-entropy fine-tuning raises top-1 agreement by .001, but it also raises the KL divergence to the target's next-token distribution from .043 to .247.

Table~\ref{tab:D2} tests whether this agreement carries over to the cache. We record the rounds of \sysname{} once and replay them with each decoder.

\begin{table}[!tb]
\centering
\small
\setlength{\tabcolsep}{3.5pt}
\caption{Offline hit rate of each decoder when the cache holds at most $B$ outcomes per round, ranked by the calibrated score of that decoder. The rounds come from \sysname{} at $J_1$-$R_1$ for Qwen3-8B with the final adapter in the setting of Table~\ref{tab:1}, at temperature 0 and budget 16. The hit rate is the share of rounds after the first of each request whose verified outcome is in the cache.}
\label{tab:D2}
\resizebox{\ifdim\width>\linewidth\linewidth\else\width\fi}{!}{\begin{tabular}{lcccccccc}
\toprule
 & \multicolumn{4}{c}{GSM8K} & \multicolumn{4}{c}{MATH-500} \\
\cmidrule(lr){2-5}\cmidrule(lr){6-9}
Decoder & $B = 1$ & $B = 4$ & $B = 8$ & $B = 16$ & $B = 1$ & $B = 4$ & $B = 8$ & $B = 16$ \\
\midrule
training-free & .624 & .927 & .977 & .993 & .634 & .903 & .961 & .989 \\
CE-trained & .617 & .934 & .979 & .994 & .654 & .935 & .980 & .995 \\
\bottomrule
\end{tabular}}
\end{table}

Fine-tuning raises the hit rate by at most .032 (the CE-trained decoder on MATH-500 at $B = 4$). We therefore serve the training-free decoder.

\subsection{Comparison of Selection Rules}\label{sec:D2}

The served rule scores each outcome $(k, a)$ by the calibrated score $\tilde P(k)\,P_\beta(a \mid k)$ of Appendix~\ref{app:details}. We compare it in Table~\ref{tab:D3} with rules that each remove one source of information. The uncalibrated rule ranks outcomes by the decoder's $P(k, a)$ of Eq.~\ref{eq:6}. The draft-model rule scores $k$ by $P(k)$ of Eq.~\ref{eq:6} with the draft model's probabilities $q_\theta(d_i)$ of its own draft tokens in place of the decoder's $\hat p_i$. It keeps the anchor factor $P_\beta(a \mid k)$ of the served rule, so it differs from the served rule only in the score of $k$. All of these rules rank anchors by the decoder.

We also compare the served rule with the geometric fan-out\ of SSD's Theorem~12, the series of Proposition~\ref{prop:2}(a) over all lengths $k = 0, \dots, \gamma$ (Appendix~\ref{app:prop2}). We fit the acceptance rate $\alpha$ and the exponent $\rho$ of Assumption~\ref{assu:2} on the calibration rounds. We then round the anchor counts $F_k$ to integers that sum to $B$\ by largest remainders. We fill the fan-out with the draft model's anchors, as SSD does. At $k = \gamma$, the draft model has no distribution for the bonus token, so the fan-out takes the decoder's anchors.\ A second row fills the same fan-out with the decoder's anchors at every length.

\begin{table}[!tb]
\centering
\small
\setlength{\tabcolsep}{3.5pt}
\caption{Offline hit rate of each selection rule when it caches at most $B$ outcomes per round, counted as in Table~\ref{tab:D2} and on the same rounds.}
\label{tab:D3}
\resizebox{\ifdim\width>\linewidth\linewidth\else\width\fi}{!}{\begin{tabular}{lcccccccc}
\toprule
 & \multicolumn{4}{c}{GSM8K} & \multicolumn{4}{c}{MATH-500} \\
\cmidrule(lr){2-5}\cmidrule(lr){6-9}
Rule & $B = 1$ & $B = 4$ & $B = 8$ & $B = 16$ & $B = 1$ & $B = 4$ & $B = 8$ & $B = 16$ \\
\midrule
served rule & .624 & .929 & .978 & .993 & .644 & .909 & .963 & .990 \\
uncalibrated rule & .613 & .921 & .975 & .993 & .585 & .870 & .944 & .981 \\
draft-model rule & .387 & .743 & .885 & .962 & .444 & .738 & .862 & .941 \\
geometric fan-out & .050 & .208 & .512 & .670 & .044 & .163 & .588 & .741 \\
geometric fan-out, decoder anchors & .067 & .289 & .655 & .890 & .059 & .222 & .682 & .898 \\
\bottomrule
\end{tabular}}
\end{table}

The served rule leads or ties every other rule at every cache size, and at $B=4$ it leads the uncalibrated rule by .008 on GSM8K and .039 on MATH-500. The draft-model rule ranks $k$ worse and trails the served rule by .186 on GSM8K and .171 on MATH-500 at $B=4$. The geometric fan-out, a fixed shape filled with the draft model's anchors, trails every rule.\ With the decoder's anchors, it still trails the served rule by .640 on GSM8K and .687 on MATH-500 at $B = 4$.

\FloatBarrier

\section{Training the Draft Model to Refine}\label{app:refine}\label{app:drafttrain}

Our draft model (Section~\ref{sec:43}) takes token inputs and trains on blocks with a clean prefix so that it can refine its own draft. We train it with a forward KL loss, and it attends causally within the block. We ablate the inputs and the training blocks in Section~\ref{sec:54} (Table~\ref{tab:3}). In this appendix, we explain that result and extend it to the final draft models (Table~\ref{tab:E2}), then ablate the share of fully noised blocks, the loss and the attention. Each draft model in these ablations is a two-epoch ablation draft model for Qwen3-8B (Appendix~\ref{app:training}).

\textbf{Inputs and training blocks.} A draft model with mask inputs drafts $J_0$ from mask embeddings and takes its own tokens as inputs from $J_1$ on. When such a draft model also trains only on fully noised blocks, its draft positions hold only mask embeddings during training, and it collapses at $J_1$ (Table~\ref{tab:3}). The $\tau$ of this draft model falls from 4.34 at $J_0$ to 1.23 at $J_1$ on GSM8K and from 5.99 to 1.28 on MATH-500. The JetSpec draft model is trained on mask inputs and fully noised blocks (Appendix~\ref{app:training}), and the \sysname{} draft model on token inputs and blocks with a clean prefix. Table~\ref{tab:E2} compares these two final draft models.

\textbf{Share of fully noised blocks.} We vary the share of fully noised blocks from 0 to 1, with token inputs (Table~\ref{tab:C-share}). The share is the probability that the clean prefix length $m$ of Eq.~\ref{eq:10} is 0. Our draft model uses 0.25 (Appendix~\ref{app:training}). Shares of 0 and 1 give a lower $\tau$ at $J_3$ than the four shares between them. From 0.125 to 0.5, $\tau$ at $J_3$ stays within 0.11 tokens on GSM8K and 0.12 on MATH-500. We decompose $\tau$ at $J_3$ into $\tau$ at $J_0$ and the gain from $J_0$ to $J_3$:
\begin{itemize}[leftmargin=*,topsep=2pt,itemsep=2pt]
\item From 0.125 to 1, $\tau$ at $J_0$ stays within 0.12 tokens on GSM8K and 0.14 on MATH-500. The draft model drafts $J_0$ from a fully noised block, and with a share of 0 it never trains on such a block. Against 0.25, a share of 0 lowers $\tau$ at $J_0$ by 0.43 tokens on GSM8K and 0.48 on MATH-500.
\item The gain from $J_0$ to $J_3$ is largest at a share of 0 (0.89 tokens on GSM8K and 1.38 on MATH-500) and falls to 0.60 and 1.02 at a share of 0.5. A share of 1 removes the clean prefixes, and the draft model does not refine.
\end{itemize}

\begin{table}[!tb]
\centering
\small
\setlength{\tabcolsep}{3.5pt}
\caption{Acceptance length $\tau$ over refinement passes for the final draft models of Qwen3-8B. All runs use the target's true states in the setting of Table~\ref{tab:F2}. Refinement passes lower the $\tau$ of the JetSpec draft model and raise that of the \sysname{} draft model.}
\label{tab:E2}
\resizebox{\ifdim\width>\linewidth\linewidth\else\width\fi}{!}{\begin{tabular}{lcccccccc}
\toprule
 & \multicolumn{4}{c}{GSM8K} & \multicolumn{4}{c}{MATH-500} \\
\cmidrule(lr){2-5}\cmidrule(lr){6-9}
Draft model & $J_0$ & $J_1$ & $J_2$ & $J_3$ & $J_0$ & $J_1$ & $J_2$ & $J_3$ \\
\midrule
JetSpec & 5.86 & 2.39 & 2.65 & 2.45 & 7.51 & 2.57 & 2.84 & 2.61 \\
JetSpec++ & 5.70 & 6.59 & 6.88 & 6.97 & 7.52 & 8.55 & 8.86 & 8.98 \\
\bottomrule
\end{tabular}}
\end{table}

\begin{table}[!tb]
\centering
\small
\setlength{\tabcolsep}{3.5pt}
\caption{Effect of the share of fully noised training blocks. Acceptance length $\tau$ for Qwen3-8B in the setting of Table~\ref{tab:3}. Each draft model takes token inputs and differs from ours only in the probability of prefix length $m = 0$, with the rest split equally over lengths 1 to 15.}
\label{tab:C-share}
\resizebox{\ifdim\width>\linewidth\linewidth\else\width\fi}{!}{\begin{tabular}{lcccccccc}
\toprule
 & \multicolumn{4}{c}{GSM8K} & \multicolumn{4}{c}{MATH-500} \\
\cmidrule(lr){2-5}\cmidrule(lr){6-9}
$\Pr[m = 0]$ & $J_0$ & $J_1$ & $J_2$ & $J_3$ & $J_0$ & $J_1$ & $J_2$ & $J_3$ \\
\midrule
0 & 3.91 & 4.56 & 4.75 & 4.80 & 5.43 & 6.39 & 6.72 & 6.81 \\
0.125 & 4.29 & 4.84 & 5.00 & 5.04 & 5.88 & 6.69 & 6.97 & 7.09 \\
0.25 (ours) & 4.34 & 4.87 & 5.01 & 5.05 & 5.91 & 6.75 & 6.99 & 7.08 \\
0.375 & 4.32 & 4.78 & 4.90 & 4.94 & 5.91 & 6.69 & 6.91 & 6.97 \\
0.5 & 4.41 & 4.84 & 4.96 & 5.01 & 6.01 & 6.74 & 6.95 & 7.03 \\
1 & 4.36 & 4.37 & 4.37 & 4.37 & 6.02 & 6.01 & 6.02 & 6.01 \\
\bottomrule
\end{tabular}}
\end{table}

\textbf{Training loss.} We compare forward KL, SFT and reverse KL as the divergence of Eq.~\ref{eq:10} (Appendix~\ref{app:training}) and change nothing else (Table~\ref{tab:4}). Forward KL leads SFT, and SFT leads reverse KL. Forward KL trains the draft model to match the target's full next-token distribution. The draft model thus also learns the target's ranking of the tokens beyond the greedy one. SFT supervises only the target's greedy token (the next token of the training sequence). The gap between SFT and forward KL is larger at $J_3$ than at $J_0$. Reverse KL trails both, and most of its deficit against forward KL at $J_3$ is already present at $J_0$ (1.40 of 1.97 tokens on GSM8K and 1.88 of 2.72 on MATH-500).

\begin{table}[!tb]
\centering
\small
\setlength{\tabcolsep}{3.5pt}
\caption{Effect of the training loss. Acceptance length $\tau$ for Qwen3-8B in the setting of Table~\ref{tab:3}.}
\label{tab:4}
\resizebox{\ifdim\width>\linewidth\linewidth\else\width\fi}{!}{\begin{tabular}{lcccccccc}
\toprule
 & \multicolumn{4}{c}{GSM8K} & \multicolumn{4}{c}{MATH-500} \\
\cmidrule(lr){2-5}\cmidrule(lr){6-9}
Loss & $J_0$ & $J_1$ & $J_2$ & $J_3$ & $J_0$ & $J_1$ & $J_2$ & $J_3$ \\
\midrule
forward KL (ours) & 4.34 & 4.87 & 5.01 & 5.05 & 5.91 & 6.75 & 6.99 & 7.08 \\
SFT & 4.28 & 4.74 & 4.85 & 4.88 & 5.81 & 6.58 & 6.76 & 6.81 \\
reverse KL & 2.94 & 3.06 & 3.08 & 3.08 & 4.03 & 4.29 & 4.34 & 4.36 \\
\bottomrule
\end{tabular}}
\end{table}

\textbf{Attention within the block.} We compare causal and bidirectional attention within the draft block and change nothing else (Table~\ref{tab:5}). At $J_0$ the bidirectional draft model matches the causal one. In the refinement passes the causal draft model leads.

\begin{table}[!tb]
\centering
\small
\setlength{\tabcolsep}{3.5pt}
\caption{Effect of the attention within the draft block. Acceptance length $\tau$ for Qwen3-8B in the setting of Table~\ref{tab:3}.}
\label{tab:5}
\resizebox{\ifdim\width>\linewidth\linewidth\else\width\fi}{!}{\begin{tabular}{lcccccccc}
\toprule
 & \multicolumn{4}{c}{GSM8K} & \multicolumn{4}{c}{MATH-500} \\
\cmidrule(lr){2-5}\cmidrule(lr){6-9}
Attention & $J_0$ & $J_1$ & $J_2$ & $J_3$ & $J_0$ & $J_1$ & $J_2$ & $J_3$ \\
\midrule
causal (ours) & 4.34 & 4.87 & 5.01 & 5.05 & 5.91 & 6.75 & 6.99 & 7.08 \\
bidirectional & 4.36 & 4.81 & 4.95 & 5.00 & 5.88 & 6.65 & 6.82 & 6.93 \\
\bottomrule
\end{tabular}}
\end{table}

\FloatBarrier

\section{Self-Refinement Example and Drafting Without Pipelining}\label{app:refreshexp}

A context refresh (Section~\ref{sec:43}) uses both the adapter and the decoder. Figure~\ref{fig:8_appendix-R1-refresh-correction} traces one context refresh and its anchor replacement, and we measure self-refinement without pipelining. Appendix~\ref{app:refresh} shows that the output stays lossless when anchor replacement rewrites the draft.

\textbf{Anchor replacement rule.} After a context refresh, the decoder reads the refreshed states of the draft $y^{(r)}$, and anchor replacement uses the first disagreement between the decoder's argmax and the draft:
\begin{equation}
\begin{gathered}
\hat k = \max\big\{\, 0 \le i \le \gamma - 1 : y^{(r)}_j = \arg\max_v \hat p(v \mid j-1) \;\, \forall\, 1 \le j \le i \,\big\},\\
\hat a = \arg\max_v \hat p(v \mid \hat k)
\end{gathered}
\label{eq:refresh}
\end{equation}
Here $\hat p(\cdot \mid j)$ are the decoder's next-token estimates of Section~\ref{sec:42}, and $(\hat k, \hat a)$ is the outcome that greedy verification of the draft would return if the decoder were the target. Anchor replacement fixes the tokens up to position $\hat k$ for the rest of the round. If $\hat a$ differs from the draft token at position $\hat k + 1$, it writes $\hat a$ there and fixes that position too. A position fixed earlier keeps its token. The next pass reads the refreshed states only up to position $\hat k$, because a token written at position $\hat k + 1$ has no refreshed state, and it redrafts only the positions after the last fixed one.

\begin{figure}[!tb]
\centering
\includegraphics[width=\textwidth]{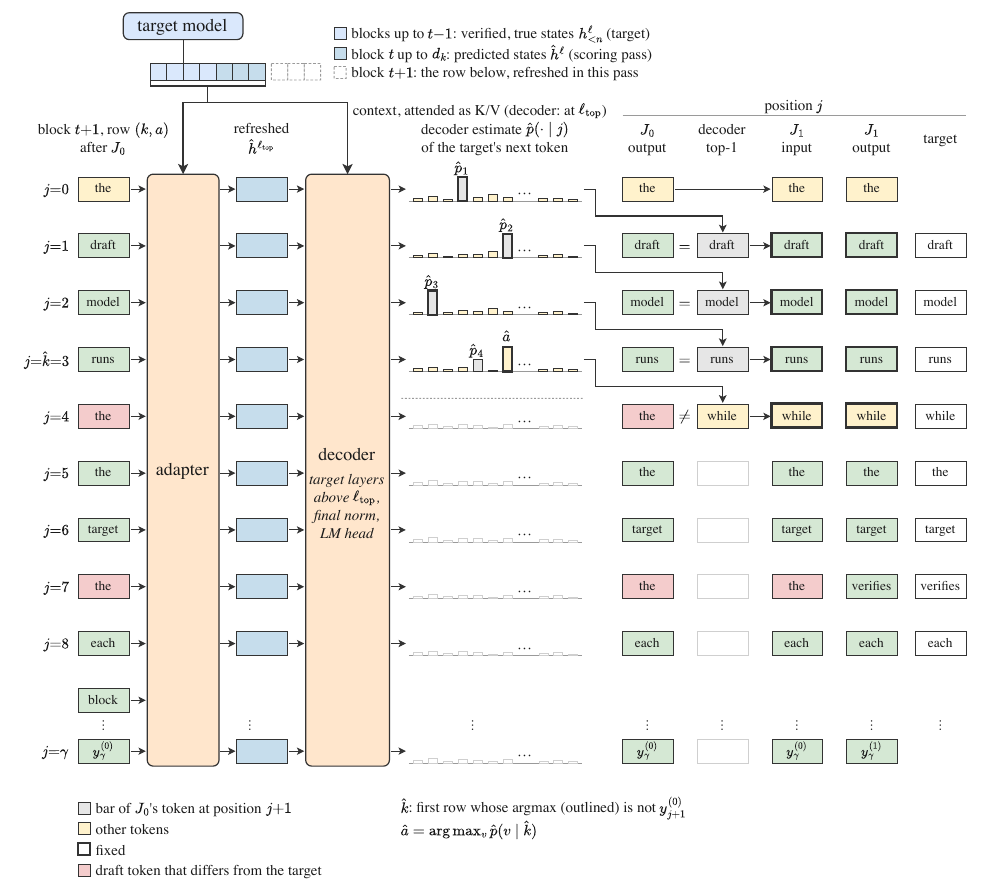}
\caption{\textbf{Correcting the first disagreement with one anchor replacement.} The figure shows the context refresh after $J_0$ of one batch row, the block $t+1$ drafted for the outcome $(k, a)$, in the layout that Figure~\ref{fig:8_appendix-A2-outcome-matrix} uses for the predicted states of block $t$. The adapter re-predicts the states of block $t+1$, and the decoder reads them as in Figure~\ref{fig:8_appendix-A2-outcome-matrix}. The decoder's top-1 token agrees with the $J_0$ draft at positions 1 to 3 and disagrees at position 4, so $\hat k = 3$ and anchor replacement writes $\hat a$ at position 4. The $J_1$ pass reads $\hat a$ at position 4 and the $J_0$ tokens at the other positions. Anchor replacement corrects position 4 and fixes positions 1 to 4, and the $J_1$ pass corrects position 7 by itself.}
\label{fig:8_appendix-R1-refresh-correction}
\end{figure}

\textbf{Drafting without pipelining.} \sysname{} drafts while the target verifies, and Table~\ref{tab:E4} compares this pipelined drafting with drafting on the critical path, which lengthens the round. When every method drafts on the critical path, $J_1$-$R_1$ leads DSpark by 0.79 tokens on GSM8K and 0.81 on MATH-500, so \sysname{} also has the higher $\tau$ without pipelining. Across $J_1$-$R_1$ to $J_3$-$R_3$, the pipelined drafting stays within 0.12 tokens of critical-path drafting on GSM8K and 0.19 on MATH-500. \sysname{} can therefore also run without pipelining when its gain in $\tau$ over DSpark outweighs any extra draft time, which favors targets whose verification takes most of the round.

\begin{table}[!tb]
\centering
\small
\setlength{\tabcolsep}{4pt}
\caption{Acceptance length $\tau$ with drafting pipelined or on the critical path, for Qwen3-8B in the setting of Table~\ref{tab:1} at budget 16 and temperature 0. A pipelined draft starts before verification ends, as \sysname{} serves. A critical-path draft starts after verification and reads the target's true states, as in DFlash and DSpark.}
\label{tab:E4}
\resizebox{\ifdim\width>\linewidth\linewidth\else\width\fi}{!}{\begin{tabular}{llcc}
\toprule
Method & Drafting & GSM8K & MATH-500 \\
\midrule
DFlash & critical path & 5.91 & 7.59 \\
DSpark & critical path & 7.31 & 9.01 \\
\midrule
\multirow{2}{*}{\sysname, $J_1$-$R_1$} & pipelined & 7.98 & 9.63 \\
 & critical path & 8.10 & 9.82 \\
\multirow{2}{*}{\sysname, $J_2$-$R_2$} & pipelined & 9.23 & 10.68 \\
 & critical path & 9.25 & 10.61 \\
\multirow{2}{*}{\sysname, $J_3$-$R_3$} & pipelined & 9.78 & 11.04 \\
 & critical path & 9.86 & 11.08 \\
\bottomrule
\end{tabular}}
\end{table}

\FloatBarrier

\section{Training and Evaluation Details}\label{app:training}

\textbf{Training data and optimization.} We take the prompts of the 800k training sequences from Nemotron-Post-Training-Dataset-v2~\citep{nvidia2025nemotron} and CodeAlpaca-20k~\citep{chaudhary2023codealpaca}, and each target model generates the responses greedily in non-thinking mode. We train the draft model of Section~\ref{sec:43} on these sequences for six epochs with AdamW and a peak learning rate of $3 \times 10^{-4}$ (linear warmup over the first 4\% of steps, then cosine decay). We use gradient clipping at norm 1.0, no weight decay, a global batch of 16 sequences and seed 42. Training takes 347 GPU-hours on GB200 GPUs for Qwen3-8B and 710 for Qwen3.8-27B. The ablation draft models are for Qwen3-8B. They use the same learning rate, learning-rate schedule, global batch and seed, but train for two epochs on the 239k mathematics sequences of the 800k set.

\textbf{Training blocks.} Training sequences have at most 3,072 tokens for Qwen3-8B and 4,096 for Qwen3.8-27B, and we sample up to 512 anchor positions per sequence. An anchor must be a supervised response token that leaves room for a full block of 16 positions (the anchor and $\gamma = 15$ draft positions). Under the prefix-length schedule, the clean prefix length $m$ of Eq.~\ref{eq:10} is 0 (a fully noised block) with probability 0.25 and each of $1, \dots, 15$ with probability 0.05. Each noise token in $\tilde y$ copies the token at an independent, uniformly random position of the sequence up to and including the anchor.

\textbf{Draft model architecture.} The draft model has five transformer layers and ends with a final normalization. It concatenates the target's hidden states $h_x$ at the five target layers of Appendix~\ref{app:adapter}, passes the concatenation through a linear layer and normalizes the result. Every draft layer projects this result into keys and values of the verified context, as in DFlash~\citep{chen2026dflash}, and attends causally within the block (Section~\ref{sec:43}). We reuse the target's embedding and LM head and keep them frozen. The draft model has 1.05B trainable parameters for Qwen3-8B and 1.73B for Qwen3.8-27B.

\textbf{Training objective.} We use the forward KL divergence at temperature 1 as the divergence of Eq.~\ref{eq:10} and average it over the supervised draft positions of a batch, with no decay along the block. In the loss ablations of Table~\ref{tab:4} (Appendix~\ref{app:drafttrain}), SFT is the cross-entropy to the next token of the training sequence, and reverse KL is $\mathrm{KL}(q_\theta \,\|\, p)$ at temperature 1.

\textbf{Baselines.} The DFlash, DSpark, JetSpec and Domino draft models are trained for six epochs on the same 800k sequences as our draft model. DFlash, DSpark, Domino and JetSpec also share our maximum length, anchor count, warmup and seed. DFlash and JetSpec have five transformer layers and read the same five target layers as our draft model. However, both baselines fill every draft position of a training block with a mask token.
\begin{itemize}[leftmargin=*,topsep=2pt,itemsep=2pt]
\item We train DFlash without the positional loss decay and at our learning rate of $3 \times 10^{-4}$. It attends bidirectionally within the block and is trained with the SFT loss at our global batch.
\item We reproduce DSpark with 15 draft tokens after the anchor, its Markov head and its confidence head. The loss also follows DSpark: a weighted sum of the cross-entropy, the L1 distance between the draft and target distributions, and the confidence loss, with weights 0.1, 0.9 and 1. DSpark weights all three terms at draft position $i$ by $\exp(-(i-1)/7)$, the positional loss decay of DFlash. We train DSpark at a learning rate of $6 \times 10^{-4}$. Its global batch is 8 for Qwen3-8B and 16 for Qwen3.8-27B.
\item Domino attends bidirectionally within the block and adds its Domino head. It is trained with its cross-entropy loss and a positional loss decay constant of 7. We train it at our learning rate of $3 \times 10^{-4}$ and at a global batch of 48.
\item JetSpec attends causally within the block and is trained with our loss and learning rate at a global batch of 8.
\item SSD uses the off-the-shelf Qwen3-0.6B model as its draft model. It drafts 4 tokens per round (a block of 5 with the anchor) and pre-drafts three outcomes $(k, a)$ for each $k$ on GSM8K and one on MATH-500, a flat fan-out as in SSD's released code.
\end{itemize}

\textbf{Benchmarks.} For the fixed shuffle of Section~\ref{sec:51}, we shuffle the test split of each benchmark with seed 0. We use the sanitized MBPP and release v6 of LCB. The runs of the ablation draft models use the first 164 prompts of GSM8K and MATH-500 from the same shuffle, at temperature 0 and with at most 256 generated tokens. On MT-Bench, we evaluate both turns of the 80 questions. Every method starts the second turn from the same first-turn answer, which the target generates with AR.

\textbf{Repetition loops.} Some outputs, mostly greedy ones at temperature 0, fall into a verbatim repetition loop that runs to the 2,048-token cap. Draft models can predict repeated text almost perfectly, which would inflate $\tau$ and the speedup. In Tables~\ref{tab:1}, \ref{tab:scale}, \ref{tab:2}, \ref{tab:6} and~\ref{tab:E4}, we therefore treat an output that reaches the cap as a repetition loop if its final 64 tokens occur verbatim at least three times in it, without overlap. We stop such an output at the onset of the loop, the first occurrence of that 64-token span, and keep other capped outputs whole. We handle this stop like an EOS stop: the round that crosses the onset counts only its tokens before the onset. We measure time through the end of that round. We cut every method, AR included, at its own onsets. On MT-Bench, the rule applies to each turn separately. The runs of Tables~\ref{tab:F2} and~\ref{tab:E2}, in which every round drafts only the block of the verified outcome, stop only at EOS or the cap.

\textbf{Metrics.} We compute $\tau$ as the committed tokens per verification round, pooled over all rounds of the requests of a benchmark, including the first round of each request. $\tau$ includes the bonus token of each round, and the token that the target produces at prefill belongs to no round. In the last round of a request, $\tau$ counts only the tokens committed before EOS or the generation cap. We compute throughput as the tokens per second, also pooled over the requests of a benchmark. Throughput counts every round after the first of each request. A round is a hit round when it verifies the pre-drafted block of the previous round's outcome, and a miss round otherwise. The hit rate is the share of hit rounds, and $\tau_{\text{hit}}$ and $\tau_{\text{miss}}$ are the values of $\tau$ on the hit and the miss rounds. These three values also exclude the first round of each request, whose block comes from no cache.

\textbf{Serving protocol.}\phantomsection\label{app:serving} For Table~\ref{tab:2}, each method generates responses to the 128 prompts of GSM8K and MATH-500 in Table~\ref{tab:1}, in four groups of 32, until EOS or 2,048 tokens. Each method runs at temperature 0 and batch size 1, with budget 16 (5 for SSD, with the settings of the Baselines paragraph above). \sysname{} runs $J_1$-$R_1$ for Qwen3-8B and $J_3$-$R_3$ for Qwen3.8-27B, with a cache of $B = 4$ outcomes (Appendix~\ref{app:details}). We pool the four groups, weighting $\tau$ by rounds and throughput by time, and the mean times exclude the first round of each request. Each speedup is relative to the AR row of the same target and benchmark.

\textbf{Correctness checks.} In exact arithmetic, verification makes the greedy output of every speculative method equal to that of AR. We run two checks on GSM8K with \sysname{} at $J_1$-$R_1$ for Qwen3-8B, as in Table~\ref{tab:2}. First, we run the target model in fp32. The first 256 tokens that \sysname{} generates for each of the first 8 prompts of the fixed shuffle (2,048 tokens in all) match AR token by token. Second, we check the CUDA graphs of the draft process against eager execution. On the first 3 prompts, \sysname{} replays 492 draft rounds from a CUDA graph, and we recompute each of them eagerly in the same round. No recomputed output differs from the graph's.

\FloatBarrier

\section{Theoretical Analysis}\label{app:theory}

In the three subsections of this appendix, we prove results that support the method of Section~\ref{sec:methods}.
\begin{itemize}[leftmargin=*,topsep=2pt,itemsep=2pt]
\item Appendix~\ref{app:prop1} shows that the decoder score $P(k,a)$ of Eq.~\ref{eq:6} is the exact probability of the verification outcome $(k,a)$ when the decoder is perfect (Section~\ref{sec:42}).
\item Appendix~\ref{app:prop2} relates the selection of Eq.~\ref{eq:8} to the geometric fan-out of SSD and shows that, with a perfect decoder, per-round selection is at least as good as any fixed allocation (Section~\ref{sec:42}).
\item Appendix~\ref{app:refresh} shows that verification stays lossless for every draft, including a draft that anchor replacement rewrites (Section~\ref{sec:43}).
\end{itemize}

\subsection{Exact Outcome Probabilities with a Perfect Decoder}\label{app:prop1}

Outcome selection (Eq.~\ref{eq:8}) ranks the outcomes $(k,a)$ by the decoder score $P(k,a)$ of Eq.~\ref{eq:6}. We show that, with a perfect decoder, $P(k,a)$ is exactly the probability that verification of the block ends at outcome $(k,a)$, at any temperature and for any draft that is fixed before verification. Eq.~\ref{eq:8} then selects the cache with the largest hit probability.

\textbf{Example.} Take $\gamma=2$. Verification uses the point-mass proposal of Appendix~\ref{app:refresh}, so it accepts each draft token with the target's probability of that token. The target gives $d_1$ probability $0.8$, and after $d_1$ it gives $d_2$ probability $0.5$ and another token $b$ probability $0.3$. Verification ends at $(1,b)$ when it accepts $d_1$ (probability $0.8$), rejects $d_2$ (probability $1-0.5=0.5$) and draws $b$ from the residual (probability $0.3/0.5=0.6$). This gives $0.8\cdot0.5\cdot0.6=0.24$, which equals $P(1,b)=\hat p_1\,\hat p(b\mid1)=0.8\cdot0.3$ for a perfect decoder. The rejection probability cancels against the normalization of the residual, and the proof of Proposition~\ref{prop:1} shows that this cancellation holds for every outcome.

\begin{assumption}[Perfect decoder]\label{assu:1}
\hfill
\begin{enumerate}[label=(\roman*),leftmargin=*,topsep=2pt,itemsep=1pt,parsep=0pt,partopsep=0pt]
\item The draft tokens $d_1,\dots,d_\gamma$ of the block are fixed before the verification of the block starts.
\item The random draws that verify this block are independent of what chose $x$ and the draft.
\item Verification uses the acceptance rule of Section~\ref{sec:prelim} with the point-mass proposal (Appendix~\ref{app:refresh}).
\item The decoder is perfect in that its estimate after prefix $j$ equals the target's next-token distribution at the sampling temperature, $\hat p(\cdot\mid j)=p(\cdot\mid x,d_{\le j})$. In this appendix, $x$ includes the anchor $y_0$, which Eq.~\ref{eq:1} lists separately, and $d_{\le 0}$ is empty, so row $0$ of the decoder estimates $p(\cdot\mid x)$, the target's distribution for $d_1$. At temperature $0$, every $p(\cdot\mid x,d_{\le j})$ is the point mass on the target's argmax, and a fixed rule breaks a tie.
\end{enumerate}
\end{assumption}

Every draft of Sections~\ref{sec:41} to~\ref{sec:43} satisfies conditions~(i) and~(ii) of Assumption~\ref{assu:1}: the initial draft $J_0$ and the fallback block of a miss, the passes of Eq.~\ref{eq:9}, drafts from predicted states, and blocks edited by anchor replacement (Appendix~\ref{app:refresh}).

Let $o^\star=(k^\star,a^\star)$ be the realized outcome of the block, where $k^\star$ is the number of accepted draft tokens and $a^\star$ is the committed bonus token. Let $P(k,a)$ be the score of Section~\ref{sec:42} computed from the draft $d_{1:\gamma}$ of the block, by Eq.~\ref{eq:6} for $k<\gamma$ and by $P(\gamma,a)=P(\gamma)\,\hat p(a\mid\gamma)$ for $k=\gamma$.

\begin{proposition}[Exact outcome distribution]\label{prop:1}
Under Assumption~\ref{assu:1} and at any temperature, $\Pr[o^\star=(k,a)\mid x,d_{1:\gamma}]=P(k,a)$ for every outcome $(k,a)$, with $a\ne d_{k+1}$ when $k<\gamma$.
\end{proposition}

\begin{proof}[Proof] Write $p_i(\cdot):=p(\cdot\mid x,d_{<i})$ for the target distribution at draft position $i$, and $\hat p_i:=\hat p(d_i\mid i-1)$. Because the proposal is the point mass on $d_i$, the acceptance rule gives
\begin{equation*}
\Pr[\text{accept }d_i\mid d_{<i}\text{ accepted}]=\min\{1,p_i(d_i)/1\}=p_i(d_i).
\end{equation*}
Each acceptance test uses a fresh draw that is independent of the draft and of the earlier tests, so the acceptance probabilities of successive positions multiply. Let $k<\gamma$, and assume first that the rejection probability at position $k+1$ is positive. Given that $d_1,\dots,d_k$ are accepted, rejection at $k+1$ has probability $1-p_{k+1}(d_{k+1})$, and the bonus token is drawn from the residual $\max\{p_{k+1}(\cdot)-\delta_{d_{k+1}}(\cdot),0\}/[1-p_{k+1}(d_{k+1})]$, where $\delta_{d_{k+1}}$ is the point mass at $d_{k+1}$. Thus, for $a\ne d_{k+1}$,
\begin{equation*}
\Pr[o^\star=(k,a)\mid x,d_{1:\gamma}]=\Big(\prod_{i=1}^{k}p_i(d_i)\Big)[1-p_{k+1}(d_{k+1})]\frac{p_{k+1}(a)}{1-p_{k+1}(d_{k+1})}=\Big(\prod_{i=1}^{k}\hat p_i\Big)\hat p(a\mid k).
\end{equation*}
The last equality applies condition~(iv) of Assumption~\ref{assu:1} to each factor. If instead $p_{k+1}(d_{k+1})=1$, then $p_{k+1}(a)=0$ for every $a\ne d_{k+1}$, and both sides vanish. For $k=\gamma$, every draft token is accepted and the bonus token is drawn from $p_{\gamma+1}(\cdot)$, so
\begin{equation*}
\Pr[o^\star=(\gamma,a)\mid x,d_{1:\gamma}]=\Big(\prod_{i=1}^{\gamma}\hat p_i\Big)\hat p(a\mid\gamma).
\end{equation*}
Both expressions equal $P(k,a)$ of Section~\ref{sec:42}.
\end{proof}

\subsection{Per-Round Selection and the Geometric Fan-Out of SSD}\label{app:prop2}

Section~\ref{sec:42} contrasts the per-round selection of Eq.~\ref{eq:8} with the fixed geometric fan-out of SSD~\citep{kumar2026speculative}, and Proposition~\ref{prop:2} states the two facts behind this contrast. Under SSD's assumptions, a continuous surrogate of Eq.~\ref{eq:8} recovers SSD's geometric series on the lengths that the surrogate's maximizer keeps (part (a)). Selection by each round's posterior is at least as good as any fixed allocation in every round (part (b)). Table~\ref{tab:D3} is the measured counterpart of part (b).

SSD caches $F_k$ anchors for each accepted length $k$ (its Definition~10). It keeps two fixed allocations and uses one of them in each round. The choice depends on whether the block under verification came from the speculation cache or from SSD's fallback after a cache miss, so it is whether the previous round hit the cache.

We model the accepted length $k$ in terms of the acceptance rate $\alpha$ of the draft model. We write the resulting law of $k$ over rounds as $P_\alpha(k)$, to keep it apart from the per-round score $P(k)$ of Eq.~\ref{eq:6}.

\begin{assumption}[Constant acceptance rate and power-law hit rate]\label{assu:2}
The draft model has constant acceptance rate $\alpha\in(0,1)$. Over rounds, each draft token is accepted with probability $\alpha$, independently given the acceptance of the tokens before it, so the realized accepted length $k$ has the law
\begin{equation*}
P_\alpha(k)=\begin{cases}\alpha^k(1-\alpha),&k<\gamma,\\\alpha^\gamma,&k=\gamma.\end{cases}
\end{equation*}

For the surrogate problem of part (a) only, the hit mass captured by $F$ anchors at length $k$, conditional on $k$, is modeled by the surrogate curve $1-F^{-\rho}$ for real $F\ge1$ and some $\rho>0$.
\end{assumption}

\begin{remark}[Source of the assumption]
The acceptance law is the model of SSD's Theorem~12, and the curve is the per-length form of SSD's power-law hit rate (its Definition~11).
\end{remark}

For part (a), we follow SSD and replace Eq.~\ref{eq:8} by a continuous allocation over lengths with the surrogate curve of Assumption~\ref{assu:2}. The curve is zero at $F=1$, so a single anchor at length $k$ captures no hit mass in the surrogate problem. The surrogate problem maximizes $\sum_k P_\alpha(k)\,(1-F_k^{-\rho})$ over $F_k\in\{0\}\cup[1,\infty)$ subject to $\sum_k F_k\le B$, where $B>1$ is the cache size of Eq.~\ref{eq:8} and the hit mass at length $k$ is $0$ when $F_k=0$. A length $k$ is active when $F_k\ge 1$.

For part (b), let $I_t$ be the information that the selection observes in round $t$, such as the outputs of the decoder and the draft model and whether the previous round hit the cache. For an outcome $o=(k,a)$ and the realized outcome $o^\star$, let
\begin{equation*}
\pi_t(o):=\Pr[o^\star=o\mid I_t],\qquad \pi_t(S):=\sum_{o\in S}\pi_t(o)\quad\text{for a cache }S.
\end{equation*}
We call $\pi_t$ the posterior of round $t$. Let $a_{k,s}$ denote the anchor at length $k$ with the $s$-th largest value of $\pi_t(k,\cdot)$, with ties broken by a fixed order on tokens. A fixed allocation $\{F_k\}$ caches $(k,a_{k,1}),\dots,(k,a_{k,F_k})$ at each length $k$. Order the outcomes by decreasing $\pi_t$, with ties broken by smaller $s$, then by smaller $k$. Let $S_t^\star$ be the set of the first $B$ outcomes in this order. This set is a solution of Eq.~\ref{eq:8} with $P$ replaced by $\pi_t$.

Suppose that $I_t$ contains the verified context $x$ and the draft $d_{1:\gamma}$, that the target's draws in round $t$ are independent of $I_t$, and that the decoder is perfect (Assumption~\ref{assu:1}). Then Proposition~\ref{prop:1} gives $\pi_t(k,a)=P(k,a)$, so $S_t^\star$ is a solution of Eq.~\ref{eq:8} itself. With an imperfect decoder, Remark~\ref{rem:est} bounds the loss of the cache of Eq.~\ref{eq:8}.

\begin{proposition}[Adaptive selection and the geometric fan-out]\label{prop:2}
\hfill
\begin{enumerate}[label=(\alph*),leftmargin=*,topsep=2pt,itemsep=2pt]
\item Under Assumption~\ref{assu:2}, every active length $k$ at any maximizer of the surrogate problem receives
\begin{equation*}
F_k=C\alpha^{k/(1+\rho)}\quad(k<\gamma),\qquad F_\gamma=C\alpha^{\gamma/(1+\rho)}(1-\alpha)^{-1/(1+\rho)}.
\end{equation*}

Here $C>0$ is the constant that makes the active $F_k$ sum to $B$, and $C=F_0$ when $\gamma\ge1$ and length $0$ is active.
\item In every round and for every cache $S_t$ chosen from $I_t$ with $|S_t|\le B$, including a fixed allocation $\{F_k\}$ with $\sum_kF_k\le B$ and a pair of fixed allocations chosen by whether the previous round hit the cache,
\begin{equation*}
\pi_t(S_t^\star)\ge \pi_t(S_t).
\end{equation*}

If $\pi_t(k,a_{k,s})$ is the same in every round, then $S_t^\star$ selects the same ranks $(k,s)$ in every round. It is therefore one fixed allocation, and the best fixed allocation ties with it.
\end{enumerate}
\end{proposition}

\textbf{Example.} The surrogate problem of part (a) can drop a length. Take $\alpha=0.6$, $\rho=1$, $\gamma=4$ and $B=10$. The series of part (a) over all lengths $0,\dots,4$ captures hit mass $0.535$. The same series, rescaled over the lengths $\{0,1,2,4\}$, captures $0.567$ and is the maximizer.

\begin{proof}[Proof of (a)] We first show that a maximizer exists and that every active $F_k$ exceeds $1$. We then show that the active lengths have equal marginal gains, which gives the stated series. Finally, we show that the active lengths are a set of $M$ lengths of largest $P_\alpha(k)$ for some $M$.

Each $F_k$ ranges over the compact set $\{0\}\cup[1,B]$ and the objective is continuous, so a maximizer exists. Because $\alpha\in(0,1)$, every $P_\alpha(k)$ is positive. At a maximizer the constraint $\sum_k F_k\le B$ binds, because raising an active $F_k$, or setting one $F_k=B$ when no length is active, strictly raises the objective. An active length with $F_k=1$ adds no hit mass. If another length is active, moving this unit to it raises the objective, and otherwise the binding constraint gives $F_k=B>1$. Hence every active $F_k$ exceeds $1$.

The marginal gain of $P_\alpha(k)(1-F_k^{-\rho})$ in $F_k$ is $P_\alpha(k)\rho F_k^{-\rho-1}$. If two active lengths had unequal marginals, moving a small amount of allocation toward the larger marginal would raise the objective. Hence the marginals of the active lengths equal a common value $\mu$, which gives
\begin{equation*}
P_\alpha(k)\rho F_k^{-\rho-1}=\mu\quad\Longrightarrow\quad F_k\propto P_\alpha(k)^{1/(1+\rho)},
\end{equation*}
and substituting $P_\alpha(k)$ from Assumption~\ref{assu:2} gives the stated series.

An active length has $F_k>1$, so $P_\alpha(k)\rho>\mu$, and a length with $P_\alpha(k)\rho\le\mu$ is therefore inactive. The converse fails: in the example above, $P_\alpha(3)\rho>\mu$ and the maximizer still leaves length $3$ inactive. If $k$ is active, $j$ is inactive and $P_\alpha(j)>P_\alpha(k)$, then setting $F_j:=F_k$ and $F_k:=0$ keeps $\sum_k F_k$ and raises the objective by $(P_\alpha(j)-P_\alpha(k))(1-F_k^{-\rho})>0$. Hence the active lengths are a set of $M$ lengths of largest $P_\alpha(k)$ for some $M$. For each $M\in\{1,\dots,\gamma+1\}$, restricting the series of part (a) to such a set gives one candidate allocation. A maximizer is a best one among the candidates that keep every $F_k\ge1$.
\end{proof}

\begin{proof}[Proof of (b)] Each $S_t=S(I_t)$ is feasible in Eq.~\ref{eq:8}, so $\pi_t(S_t^\star)\ge\pi_t(S_t)$ pointwise, and the inequality also holds in expectation over rounds. If $\pi_t(k,a_{k,s})$ does not depend on $t$, the order that defines $S_t^\star$ selects the same ranks $(k,s)$ in every round. Because $\pi_t(k,a_{k,s})$ is nonincreasing in $s$ and ties prefer smaller $s$, the selected ranks at length $k$ are the top $F_k$ ranks, which form a fixed allocation.
\end{proof}

\begin{remark}[Relation of the proposition to SSD]
The allocation of part (a) is the geometric series of SSD restricted to the active lengths, and it equals the fan-out of SSD's Theorem~12 when every length $k=0,\dots,\gamma$ is active. SSD derives the equal-marginal condition of part (a) in its Theorem~18, but its statement has no inactive lengths. SSD fills each length with the top anchors of the draft model's logits. The resulting cache is chosen from $I_t$, so part (b) covers it. This cache can fall strictly below $S_t^\star$.
\end{remark}

\begin{remark}[Estimated scores]\label{rem:est}
In serving, the draft process ranks outcomes by an estimate $\hat\pi_t$ of $\pi_t$. For any cache $\hat S_t$ of at most $B$ outcomes, let
\begin{equation*}
\nu_t:=\max_{|S|\le B}\big|\pi_t(S)-\hat\pi_t(S)\big|, \qquad \zeta_t:=\max_{|S|\le B}\hat\pi_t(S)-\hat\pi_t(\hat S_t).
\end{equation*}
Then $\pi_t(S_t^\star)-\pi_t(\hat S_t)\le2\nu_t+\zeta_t$ in every round. The bound holds because $\nu_t$ bounds the estimation error at both $S_t^\star$ and $\hat S_t$, and $\zeta_t$ bounds the estimated gap between $S_t^\star$ and $\hat S_t$.
\end{remark}

\subsection{Lossless Verification}\label{app:refresh}

This subsection shows why the output stays lossless for every draft, including one that anchor replacement rewrites.

\sysname{} determines every draft token before verification. Each token is an argmax of the draft model, from the initial draft $J_0$ (also the fallback block of a miss) or a pass of Eq.~\ref{eq:9}, or it is the decoder's token that anchor replacement writes. The noise tokens that $J_0$ reads are drawn independently of the target's draws in the round that verifies the block.

Given the draft, the proposal at position $i$ is the point mass $q_i=\delta_{d_i}$. Let $p_i=p(\cdot\mid x,d_{<i})$ be the target's distribution at the sampling temperature. With this proposal, the acceptance rule of Section~\ref{sec:prelim} accepts $d_i$ with probability $p_i(d_i)$. On a rejection, the rule samples from $p_i$ with $d_i$ removed and renormalized. When $p_i(d_i)=1$, a rejection has probability $0$ and the rule never uses this residual. This rule is equivalent to drawing $v\sim p_i$, accepting $d_i$ if and only if $v=d_i$, and committing $v$ otherwise. DARTree~\citep{li2026dartree} describes this exact-match verification for DFlash~\citep{chen2026dflash}. Every committed token is thus the target's own draw, at any temperature and whichever model produced the draft token.

Using the draft model's softmax as $q$ for an argmax token would break this guarantee, because the token is not drawn from that softmax. For example, if the softmax puts $0.9$ on $d_i$ and the target puts $0.6$ on it, the rule with that $q$ commits $d_i$ with probability $2/3$ instead of $0.6$.

\end{document}